\documentclass{article}

\usepackage{microtype}
\usepackage{graphicx}
\usepackage{subfigure}
\usepackage{booktabs} 
\usepackage{enumitem}

\usepackage{hyperref}

\usepackage[accepted]{icml2024}

\usepackage{amsmath}
\usepackage{amssymb}
\usepackage{mathtools}
\usepackage{amsthm}
\usepackage{bbm} 
\usepackage{bm} 
\usepackage{multicol}

\RequirePackage{algorithm}
\RequirePackage{algorithmic}

\usepackage[capitalize,noabbrev]{cleveref}

\theoremstyle{plain}
\newtheorem{theorem}{Theorem}[section]

\newtheorem{lemma}[theorem]{Lemma}
\newtheorem{corollary}[theorem]{Corollary}
\theoremstyle{definition}
\newtheorem{definition}[theorem]{Definition}

\theoremstyle{remark}

\usepackage[textsize=tiny]{todonotes}

\usepackage{xspace}

\newcommand{\COMID}{\textup{\textbf{\texttt{COMID}}}\xspace}
\newcommand{\COMIDA}{\textup{\textbf{\texttt{COMIDA}}}\xspace}
\newcommand{\COMIDAMDP}{\textup{\textbf{\texttt{COMIDA-MDP}}}\xspace}
\newcommand{\COGDA}{\textup{\textbf{\texttt{COGDA}}}\xspace}

\newcommand{\xt}{\xvec_t}
\newcommand{\yt}{\yvec_t}

\renewcommand{\vec}[1]{{\boldsymbol{{#1}}}}

\newcommand{\Em}{\bm{E}}
\newcommand{\Pm}{\bm{P}}
\newcommand{\xvec}{\bm{x}}
\newcommand{\xivec}{\bm{\xi}}
\newcommand{\yvec}{\bm{y}}
\newcommand{\bvec}{\bm{b}}
\newcommand{\cvec}{\bm{c}}
\newcommand{\Mm}{\bm{M}}
\newcommand{\vvec}{{\bm{v}}}
\newcommand{\uvec}{\bm{u}}

\newcommand{\uhat}{\wh{\uvec}}

\newcommand{\muvec}{\bm{\mu}}

\newcommand{\nuvec}{\bm{\nu}}
\newcommand{\gmu}{\bm{g}_{\mu}(t)}
\newcommand{\gvec}{\bm{g}}

\newcommand{\gx}{\bm{g}_{x}(t)}
\newcommand{\gy}{\bm{g}_{y}(t)}

\newcommand{\Hu}[1]{\bm{h}_{\bm{u}}(#1)}

\newcommand{\gu}{\bm{g}_{\bm{u}}(t)}

\newcommand{\gtx}{\wt{\bm{{g}}}_{x}(t)}
\newcommand{\gtu}{\wt{\bm{{g}}}_{u}(t)}
\newcommand{\gty}{\wt{\bm{{g}}}_{y}(t)}
\newcommand{\gtv}{\wt{\bm{{g}}}_{v}(t)}
\newcommand{\gtmu}{\wt{\bm{{g}}}_{\mu}(t)}

\newcommand{\oxvec}{\overline{\xvec}}
\newcommand{\oyvec}{\overline{\yvec}}

\newcommand{\omuvec}{\overline{\muvec}}
\newcommand{\ovvec}{\overline{\vvec}}

\newcommand{\freg}{f^{\pa{\text{reg}}}}
\newcommand{\lt}{\ell_{t}}
\newcommand{\ltreg}{\lt^{\pa{\text{reg}}}}

\newcommand{\Y}{\mathcal{Y}}
\newcommand{\X}{\mathcal{X}}
\newcommand{\U}{\mathcal{U}}
\newcommand{\F}{\mathcal{F}}
 \newcommand{\A}{\mathcal{A}}

\newcommand{\real}{\mathbb{R}}
\newcommand{\Rn}[0]{\mathbb{R}} 

\newcommand{\Sw}{\mathcal{S}}
\newcommand{\Atot}{SA}
\newcommand{\AAtot}{\Sw\times \A}

\newcommand{\HH}{\mathcal{H}}

\newcommand{\OO}{\mathcal{O}}

\newcommand{\II}[1]{\mathbb{I}_{\left\{#1\right\}}}
\newcommand{\LL}{\mathcal{L}}

\newcommand{\Lreg}{\LL^{\pa{\text{reg}}}}

\newcommand{\PP}[1]{\mathbb{P}\left[#1\right]}

\newcommand{\EE}[1]{\mathbb{E}\left[#1\right]}

\newcommand{\EEt}[1]{\mathbb{E}_t\left[#1\right]}
\newcommand{\EEpi}[1]{\mathbb{E}_\pi\left[#1\right]}
\newcommand{\EEcpi}[2]{\mathbb{E}_\pi\left[#1\middle|#2\right]}

\newcommand{\EEc}[2]{\mathbb{E}\left[#1\left|#2\right.\right]}

\def\argmin{\mathop{\mbox{ arg\,min}}}
\def\argmax{\mathop{\mbox{ arg\,max}}}
\newcommand{\ra}{\rightarrow}

\newcommand{\iprod}[2]{\left\langle#1,#2\right\rangle}

\newcommand{\norm}[1]{\left\|#1\right\|}
\newcommand{\abs}[1]{\left|#1\right|}

\newcommand{\twonorm}[1]{\norm{#1}_2}
\newcommand{\sqtwonorm}[1]{\norm{#1}_2^{2}}
\newcommand{\sqopnorm}[2]{\norm{#1}_{#2}^{2}}
\newcommand{\infnorm}[1]{\norm{#1}_\infty}
\newcommand{\spannorm}[1]{\norm{#1}_{\text{sp}}}

\newcommand{\ev}[1]{\left\{#1\right\}}
\newcommand{\pa}[1]{\left(#1\right)}
\newcommand{\bpa}[1]{\bigl(#1\bigr)}
\newcommand{\Bpa}[1]{\Bigl(#1\Bigr)}

\newcommand{\dom}[1]{\text{dom}\pa{#1}}

\newcommand{\wh}{\widehat}
\newcommand{\wt}{\widetilde}

\newcommand{\transpose}{^\mathsf{\scriptscriptstyle T}}

\definecolor{PalePurp}{rgb}{0.66,0.57,0.66}

\newcommand{\DDKL}[2]{\mathcal{D}_{\textup{KL}}\pa{#1\middle\|#2}}
\newcommand{\DDf}[3]{\mathcal{D}_{#1}\pa{#2\middle\|#3}}

\newcommand{\DDx}[2]{\mathcal{D}_{x}\!\pa{#1\middle\|#2}}
\newcommand{\DDy}[2]{\mathcal{D}_{y}\!\pa{#1\middle\|#2}}

\newcommand{\DDu}[2]{\mathcal{D}_{u}\!\pa{#1\middle\|#2}}

\newcommand{\VV}{\mathcal{V}}
\newcommand{\HHx}{\HH_x}
\newcommand{\HHy}{\HH_y}
\newcommand{\HHv}{\HH_v}
\newcommand{\HHu}{\HH_u}

\newcommand{\R}{\mathbb{R}}

\DeclarePairedDelimiterX\ip[2]{\langle}{\rangle}{#1,#2}

\let\P\undefined
\DeclarePairedDelimiterXPP\P[1]{\mathbb{P}}(){}{
    
    #1
}

\DeclarePairedDelimiterXPP\E[1]{\mathbb{E}}[]{}{
    
    #1
}

\DeclarePairedDelimiterXPP\Es[2]{\mathbb{E}_{#1}}[]{}{
    
    #2
}

\newcommand{\rvec}{\bm{r}}

\DeclarePairedDelimiterXPP\Ipt[2]{{#1}\transpose}(){}{#2}
\DeclarePairedDelimiterXPP\Iptr[2]{}(){\transpose{#2}}{#1}

\newcommand{\subjectto}{\mathrm{subject\ to}}

\icmltitlerunning{Dealing with unbounded gradients in stochastic saddle-point optimization}

\begin{document}
\allowdisplaybreaks
\twocolumn[
\icmltitle{Dealing with unbounded gradients in stochastic saddle-point 
optimization}



\icmlsetsymbol{equal}{*}

\begin{icmlauthorlist}
\icmlauthor{Gergely Neu}{upf}
\icmlauthor{Nneka Okolo}{upf}
\end{icmlauthorlist}

\icmlaffiliation{upf}{Universitat Pompeu Fabra, Barcelona, Spain}

\icmlcorrespondingauthor{Gergely Neu}{gergely.neu@gmail.com}
\icmlcorrespondingauthor{Nneka Okolo}{nnekamaureen.okolo@upf.edu}

\icmlkeywords{saddle-point optimization, bilinear games, convex optimization, stochastic optimization}

\vskip 0.3in
]



\printAffiliationsAndNotice{}  

\renewcommand{\hat}{\wh}

\begin{abstract}
	We study the performance of stochastic first-order methods for finding 
	saddle points of convex-concave 
	functions. A notorious challenge faced by such methods is that the 
	gradients can grow arbitrarily large during 
	optimization, which may result in instability and divergence. In this 
	paper, we propose a simple and effective 
	regularization technique that stabilizes the iterates and yields meaningful 
	performance guarantees even if the domain 
	and the gradient noise scales linearly with the size of the iterates (and 
	is thus potentially unbounded). Besides 
	providing a set of general results, we also apply our algorithm to a 
	specific 
	problem in reinforcement 
	learning, where it leads to performance guarantees for finding near-optimal 
	policies in an average-reward MDP without 
	prior knowledge of the bias span.
\end{abstract}

\section{Introduction}\label{sec:intro}
We study the performance of stochastic optimization algorithms for solving convex-concave saddle-point problems of the 
form $\min_{\xvec\in\X}\max_{\yvec\in\Y}f(\xvec,\yvec)$. The algorithms we consider aim to approximate saddle points 
via running two stochastic convex optimization methods against each other, one aiming to minimize the objective function 
and the other aiming to maximize. Both players have access to noisy gradient evaluations at individual points 
$\xvec_t$ and $\yvec_t$ of the primal and dual domains $\X$ and $\Y$, and typically compute their updates via 
gradient-descent-like procedures. Due to the complicated interaction between the two concurrent procedures, it is 
notoriously difficult to ensure convergence of these methods towards the desired saddle points, and in fact even 
guaranteeing their stability is far from trivial. One common way to make sure that the iterates do not diverge is 
projecting them to bounded sets around the initial point. While this idea does the job, it gives rise to a dilemma: how 
should one pick the size of these constraint sets to make sure both that the optimum remains in there while keeping the 
optimization process reasonably efficient? In this paper, we propose a method 
that addresses this question and provides as good guarantees as the best known projection-based method, but without 
having to commit to a specific projection radius.

It is well known that simply running gradient descent for both the minimizing and maximizing players can easily result 
in divergence, even when having access to exact gradients without noise 
\citep{goodfellow2016nips,mertikopoulos2018optimistic}. While the average of the iterates 
may converge in such cases, their rate of convergence is typically affected by the magnitude of the gradients, which 
grows larger and larger as the iterates themselves diverge, thus resulting in arbitrarily slow convergence of the 
average. Numerous solutions have been proposed to this issue in the literature, most notably using some form of 
\emph{gradient extrapolation} 
\citep{korpelevich1976extragradient,popov1980modification,gidel2018variational,mertikopoulos2018optimistic}. When these 
methods 
have access to noiseless gradients and are run on smooth objectives, these methods are remarkably stable: they can be 
shown to converge monotonically towards their limit. That said, convergence of such methods in the stochastic case is 
much less well-understood, unless the iterates are projected to a compact set 
\citep{juditsky2011solving,gidel2018variational}, or the boundedness of the gradients ensured by other assumptions 
\citep{mishchenko2020revisiting,loizou2021stochastic,sadiev2023high}. Indeed, unless projections are employed, the 
iterates of one player may grow large, which can result in large gradients observed by the opposite player, which in 
turn may result in large iterates for the second player---which effects may eventually cascade and result in instability 
and divergence. Our main contribution is proposing a stabilization 
technique that 
eliminates the risk of divergence.

For the sake of exposition, let us consider the special case of bilinear objectives
\[
	f(\xvec,\yvec) = \xvec\transpose 
	\Mm\yvec + 
	\bvec\transpose\xvec - \cvec\transpose\yvec,
\]
and primal-dual stochastic gradient descent ascent starting from the initial 
point 
$\xvec_1 = 0$ and $\yvec = 0$ as a baseline:
\begin{align*}
\xvec_{t+1} &= \xvec_t - \eta  \wt{\gvec}_x(t)\\
\yvec_{t+1} &= \yvec_t - \eta  \wt{\gvec}_y(t),
\end{align*}
where $\wt{\gvec}_x(t) = \nabla_x f(\xvec_t,\yvec_t) + \xi_x(t)$ and $\wt{\gvec}_y(t) = 
\nabla_y f(\xvec_t,\yvec_t) + \xi_y(t)$ are potentially noisy 
and unbiased estimates of the gradients with respect to $\xvec$ and 
$\yvec$.
Using standard tools (that we will explain in detail below), the average of the first $T$ iterates produced by the 
above procedure can be shown to satisfy the following guarantee on the 
\emph{duality gap}:
\begin{align*}
 &G(\xvec^*,\yvec^*) = f\pa{\frac{1}{T}\sum_{t=1}^T \xvec_t, \yvec^*} - f\pa{\xvec^*,\frac{1}{T}\sum_{t=1}^T \yvec_t}
 \\
 &\,\le \frac{\twonorm{\xvec^*}^2 + \twonorm{\yvec^*}^2}{2\eta T} 
 + \frac{\eta}{2T} \sum_{t=1}^T 
{\EE{\twonorm{\wt{\gvec}_x(t)}^2 + \twonorm{\wt{\gvec}_y(t)}^2}}.
\end{align*}
If one can ensure that the gradient estimates $\wt{\gvec}_x(t)$ and $\wt{\gvec}_y(t)$ remain bounded by a constant $G > 
0$, one can 
set $\eta \sim 1/(G\sqrt{T})$ and obtain a convergence rate of 
order $\frac{G \pa{\twonorm{\xvec^*}^2 + 
\twonorm{\yvec^*}^2}}{\sqrt{T}}$. However, notice that \emph{there is no way to make sure that the gradients actually 
remain bounded} while executing the algorithm! Indeed, notice that 
$\nabla_x f(\xvec_t,\yvec_t) = \Mm \yvec_t + 
\bvec$, which grows large as $\yvec_t$ grows large. A natural idea  is to 
project the iterates to balls of 
respective sizes $D_x,\,D_y > 0$, which guarantees that the iterates and thus 
the gradients remain bounded. 
However, convergence to the saddle point $(\xvec^*,\yvec^*)$ is now only possible whenever the respective norms satisfy 
$\norm{\xvec^*} \le D_x$ and $\norm{\yvec^{*}}\le D_y$, otherwise the optimal 
solution is excluded from the 
feasible set. Unfortunately, in many applications, it is impossible to pick the 
constants $D_x$ and $D_y$ 
appropriately due of lack of prior knowledge of the solution norms. 

In this paper, we propose a method that guarantees upper bounds on the duality gap of the following form (when 
specialized to the setting described above):
\[
 G(\xvec^*,\yvec^*) = \OO\pa{\frac{\twonorm{\xvec^*-\xvec_1}^2 
+ \twonorm{\yvec^*-\yvec_1}^2 + 1}{\sqrt{T}}},
\]
where the big-O notation hides some problem-dependent constants related to the objective function $f$ (which will be 
made explicit in our main theorem). Notably, our method requires no prior knowledge of the norms 
$\twonorm{\xvec^*-\xvec_1}^2$ and $\twonorm{\yvec^*-\yvec_1}^2$ whatsoever, and in particular performs no projections 
to make sure that the iterates remain bounded, thus addressing the challenge 
outlined above. Furthermore, our 
guarantees continue to hold for potentially data-dependent comparators $(\xvec^*,\yvec^*)$, and even when the gradients 
are subject to \emph{multiplicative noise} that can scale with the magnitude of 
the gradient itself.  Our main 
technical tool is augmenting the objective with a well-chosen regularization term which 
allows us to eliminate the terms $\twonorm{\wt{\gvec}_x(t)}^2 + \twonorm{\wt{\gvec}_y(t)}^2$ appearing in the guarantee 
of 
standard primal-dual gradient descent, and replace them with an upper bound of the gradients of the objective evaluated 
at the initial point $(\xvec_1,\yvec_1)$. These bounds have the appealing 
property of being 
\emph{initialization-dependent}, in that they guarantee improved performance 
when we pick the initial points 
$(\xvec_1,\yvec_1)$ close to $(\xvec^*,\yvec^*)$. Besides the simple 
bilinear setting discussed above, we study a 
more general set of games that we call ``sub-bilinear'', and provide an algorithmic framework that comes with 
guarantees that express closeness to initialization in terms of general Bregman 
divergences. 

The initialization-adaptive nature of our bounds is similar to the guarantees proved by 
\citet{liu2022initialization}, who propose algorithms for the same setting that achieve an 
initialization-dependent convergence rate of the order $G(\xvec^*,\yvec^*) = 
\wt{\OO}\pa{\frac{\twonorm{\xvec^*-\xvec_1} 
+ \twonorm{\yvec^*-\yvec_1}}{\sqrt{T}}}$, where $\wt{\OO}(\cdot)$ hides 
polylogarithmic factors of $T$ and the comparator norms. While this guarantee may appear 
to be stronger than our result, it is only proved under the condition that all gradients remain 
bounded as $\twonorm{\gvec_x(t)}^2 \le 1$ and $\twonorm{\gvec_y(t)}^2\le 1$, which cannot be 
ensured in the unconstrained setting that we consider in the present work. 
Another closely related work (that was pointed to our attention after posting 
the initial version of this work 
online) is due to \citet{jacobsen2023unconstrained}, who proposed an algorithm for convex optimization that is able to 
deal with potentially unbounded gradients that satisfy a certain ``sub-quadratic'' growth condition. As their Section~3 
shows, their technique is directly extensible to saddle-point optimization in bilinear games and even our notion of 
sub-bilinear games. Even more curiously, their key technical idea is nearly identical to the one employed in our work. 
However, the objectives of their work were slightly different from ours in that they focused on the special case of 
Euclidean geometries and deterministic gradients. In contrast, our main results hold for stochastic gradients and 
stochastic comparator points $(\xvec^*,\yvec^*)$, and are expressed in terms of general Bregman divergences. We believe 
that their analysis can be adjusted to handle these extensions: while dealing with stochastic gradients seems easy, 
handling data-dependent comparator points and working with general Bregman divergences appears to be rather 
challenging in their framework, and could require significant changes to their algorithm and analysis. We nevertheless 
encourage the reader to credit \citet{jacobsen2023unconstrained} just as much as our work when referring to 
the regularization technique we study in this paper.


More broadly speaking, our work contributes to the line of work on 
\emph{parameter-free} optimization methods that are 
able to adapt to problem complexity without prior knowledge of the relevant problem parameters. In the context of 
online convex optimization (OCO), several effective parameter-free algorithms are known to achieve guarantees 
scaling optimally with the initialization error $\norm{\xvec^* - \xvec_1}$, without requiring prior knowledge 
thereof \citep{streeter2012no,orabona2013dimension,van2019user}. 
\citet{cutkosky2016online, cutkosky2019artificial,mhammedi2020lipschitz} 
improve these guarantees by providing 
initialization-adaptive bounds for OCO in unconstrained domains without prior knowledge of both the size of the domain 
or subgradients of the loss. 
One would think that this would make their method suitable for solving the problem we study in this paper---however, 
their bound depends on the maximum norm of the observed subgradients, which is problematic for the reasons we have 
discussed extensively above. 


Unconstrained saddle-point problems have many important applications. Perhaps the most well-known such applications are 
in optimizing dual representations of convex functions 
(\citealp{shalev2006convex}; 
\citealp[Section~5.2]{bubeck2015convex};\citealp{wang2018acceleration,wang2023no}).
 Our original 
motivation during the development of this work has been to develop primal-dual methods for solving average-reward Markov 
decision processes (MDPs): this problem can be formulated as a linear program with primal variables that are of unknown 
scale. In the simpler setting of discounted Markov decision processes, previous work has provided efficient planning 
methods based on saddle-point optimization \citep{MWang2017,jin2020efficiently,cheng2020reduction}. While in this 
simple setting the primal variables (called \emph{value functions} in this setting) are known to be uniformly bounded, 
this is not the case in the more challenging average-reward setting we consider here: in this case, the value functions 
can have arbitrarily large norm depending on the program structure. 
As it is well-known in the reinforcement-learning 
literature, estimating this parameter with only online sample access 
is as hard 
as solving the original 
problem, and learning optimal policies without 
its prior knowledge has been widely conjectured to be impossible 
\citep{bartlett09regal,fruit2018efficient,fruit2018near,zhang2019regret}. 
Using our techniques developed in the present 
paper, we make progress on this important problem by proposing a planning algorithm that is guaranteed to produce a 
near-optimal policy without having prior knowledge of the scale of the value functions after a polynomial number of 
queries made to a simulator of the environment.

\paragraph{Notations.}
For an integer $T$, we use $[T] = 1,2,\cdots,T$. We denote as
$\vec{1}$ the vector with all one entries in $\Rn^{m}$ and represent the 
positive orthant as $\Rn_{+}^{n}$. Let $\X\subseteq\Rn^{m}$ and $f : 
\X\rightarrow\R$ differentiable. For vectors $\xvec,\xvec'\in\X$ we define 
their inner product as $\iprod{\xvec}{\xvec'} = 
\sum_{i=1}^{m}\xvec_{i}\xvec'_{i}$ and the \emph{Bregman divergence} of $\xvec$ 
at $\xvec'$ induced by $f : \X\rightarrow\R$ as 
\[
\DDx{\xvec}{\xvec'} = f(\xvec) - f(\xvec') - \iprod{\nabla f(\xvec')}{\xvec - 
\xvec'}.
\]

\section{Preliminaries}\label{sec:prelim}
We now formally define our problem setup and objectives. First, we recall some standard definitions.
\begin{definition}\label{def:convx_concv}(\textbf{convex-concave function})
	Let $\X\subseteq\Rn^{m}, \Y\subseteq\Rn^{n}$ be convex 
	sets. A function $f : \X\times \Y \rightarrow \Rn$ is said to be 	
	\emph{convex-concave} if it is convex in the first argument and concave in 
	the second. That is, $f$ is convex-concave if for any $\xvec,\xvec'\in\X$, 
	$\yvec,\yvec'\in\Y$ and 
	$\lambda\in[0,1]$, we have
	\begin{equation*}
		f\pa{\lambda\xvec + (1-\lambda)\xvec',\yvec}
		\leq \lambda f\pa{\xvec,\yvec} + (1-\lambda)f\pa{\xvec',\yvec},
	\end{equation*}
	and 
	\begin{equation*}
		f\pa{\xvec,\lambda\yvec + (1-\lambda)\yvec'}
		\geq \lambda f\pa{\xvec,\yvec} + (1-\lambda)f\pa{\xvec,\yvec'}.
	\end{equation*}
\end{definition}
\begin{definition}\label{def:subgrad}(\textbf{subgradient and subdifferential})
	Let $\X^{*}$ denote the dual space of $\X$. For a function $h : \X 
	\rightarrow \Rn$, $\vec{g}\in\X^{*}$ is a 
	subgradient of $h$ at $\xvec\in\X$ 	if for all $\xvec'\in\X$,
	\begin{equation*}
		h(\xvec) - h(\xvec') \leq \iprod{\vec{g}}{\xvec - \xvec'}.
	\end{equation*}
	The set of all subgradients of a function $h$ at $\xvec$ is called the \emph{subdifferential}, and is denoted 
by $\partial h(\xvec)$.
\end{definition}
We recall that when $h$ is convex and differentiable, then $\nabla h(\xvec) \in \partial h(\xvec)$ holds for all 
$\xvec\in\X$, and additionally $\partial h(\xvec) = \ev{\nabla h(\xvec)}$ holds whenever $\xvec$ is in the interior of 
the domain $\X$.
\begin{definition}\label{def:scovx}(\textbf{strong convexity}) 
	For $\gamma\geq 0$, a function $h : \X \rightarrow \Rn$ is 
	$\gamma$-strongly convex with respect to the norm $\norm{\cdot}$ if and only if for 
	all $\xvec,\xvec'\in\dom{h}$, $\vec{g}\in\partial h(\xvec)$:
	\begin{equation*}
		h(\xvec') - h(\xvec) \geq \iprod{\vec{g}}{\xvec' - \xvec} + 
		\frac{\gamma}{2}\norm{\xvec' - \xvec}^{2}.
	\end{equation*}
\end{definition}

We consider the problem of finding (approximate) saddle points of convex-concave functions on the potentially unbounded 
convex domains $\X\times\Y \subseteq \real^m\times \real^n$:
\begin{equation}\label{eq:min_max_prb}
\inf_{\xvec\in\X}\sup_{\yvec\in\Y}f(\xvec,\yvec),
\end{equation}
where $f : \X\times \Y \rightarrow \Rn$ is assumed to be convex-concave in the 
sense of
\cref{def:convx_concv}. We focus on the classic stochastic first-order oracle 
model where algorithms can only access 
noisy estimates of the subgradients at individual points in $\X\times\Y$. Specifically, we will consider incremental 
algorithms that produce a sequence of iterates $\pa{\xvec_t,\yvec_t}_{t=1}^T$ by running two concurrent online learning 
methods for choosing the two sequences $\ev{\xvec_t}_{t=1}^{T}$ and 
$\ev{\yvec_t}_{t=1}^T$. The algorithm picking the 
sequence $\ev{\xvec_t}_{t=1}^{T}$ aims to minimize the sequence of losses 
$\ev{f(\cdot,\yvec_t)}_{t=1}^T$ and is 
referred to as the \emph{min player}, and the algorithm picking 
$\ev{\yvec_t}_{t=1}^{T}$ that aims to minimize 
$\ev{-f(\xvec_t,\cdot)}_{t=1}^T$ is called the \emph{max player}. In each round 
$t$, the two players have access to a 
stochastic first-order oracle that provides the following noisy estimates of a pair of subgradients $\gx 
\in \partial_{x} f(\xvec_t,\yvec_t)$ and $\gy \in -\partial_{y} \pa{- f(\xvec_t,\yvec_t)}$, with 
the noisy estimates written as \looseness=-1 
\begin{align*}
	\gtx &= 
	\gx + \xi_{x}(t)\\
	\gty &= 
	\gy + \xi_{y}(t).
\end{align*}
Here, $\xivec_{x}(t)\in\real^m$ and $\xivec_{y}(t)\in\real^n$ are zero-mean noise vectors generated in round $t\in[T]$ 
from some unknown distributions, independently of the interaction history $\F_{t-1}$.  Using the 
notation $\EEt{\xvec_{t}} = \EEc{\xvec_{t}}{\F_{t-1}} = \xvec_{t}$ to denote expectations conditioned on 
the history of observations up to the end of time $t$, we can write the above conditions as $\EEt{\gtx} = \gx$ and 
$\EEt{\gty} = \gy$. Note that when the objective is differentiable we can simply set $\gx = \nabla_{x} 
f(\xvec_{t},\yvec_t)$ and $\gy = \nabla_{y} f(\xvec_{t},\yvec_t)$. 

In a good part of this work, we focus on the important class of \emph{bilinear} objective functions that take the 
following 
form:
\[
 f(\xvec,\yvec) = \xvec\transpose \Mm\yvec + \bvec\transpose\xvec - \cvec\transpose\yvec.
\]
Here, $\Mm\in\real^{m\times n}$, $\bvec\in\real^m$ and $\cvec\in\real^n$ and the domains for the optimization 
variables are $\X = \real^m$ and $\Y = \real^n$. This objective is clearly differentiable, and its gradients with 
respect to $\xt$ and $\yt$ are given as $\gx = \Mm\yvec_t + \bvec$ and $\gy =\Mm\transpose\xvec_t - \cvec$. In the 
context 
of bilinear games, we will consider a natural noise model where in each round $t$, we have access to noisy versions of 
the matrices and vectors necessary for computing the gradients. Specifically, we have $\wh{\Mm}(t) = \Mm + 
\xivec_{\Mm}(t)$, $\wh{\bvec}(t) = \bvec + \xivec_{b}(t)$, and $\wh{\cvec}(t) = \cvec + \xivec_{c}(t)$ 
where $\xivec_{\Mm}(t), \xivec_{b}(t), \xivec_{c}(t)$ are $i.i.d$, zero-mean 
random matrices and vectors generated from unknown distributions. We then use these observations to build the following 
estimators for the gradients:
\begin{align*}
	\gtx
	&= \wh{\Mm}(t)\yvec_{t} + \wh{\bvec}(t)\\
	\gty
	&= \wh{\Mm}(t)\transpose \xvec_{t} - 
	\wh{\cvec}(t).
\end{align*}
This fits into the generic noise model defined earlier with $\xivec_{x}(t) = \xivec_{\Mm}(t)\yt + \xivec_{b}(t)$ and 
$\xivec_{y}(t) = \xivec_{M}(t)\transpose \xt - \xivec_{c}(t)$. Regarding the magnitude of the noise, we will make the 
assumption that there exists constants $L_M$, $L_b$ and $L_c$ such that 
$\mathbb{E}_t\bigl[\|\wh{\bvec_t}\|^2\bigr] \le L_b$, 
$\mathbb{E}_t\bigl[\twonorm{\wh{\cvec}_t}^2\bigr] \le L_c$, and 
\begin{align*}
\EEt{\bigl\|\wh{\Mm}(t)\yvec\bigr\|_2^2} &\leq 		
L_M^2{\sqtwonorm{\yvec}}
\\
\EEt{\bigl\|\wh{\Mm}(t)\transpose\xvec\|_2^2} &\leq 
L_M^2{\sqtwonorm{\xvec}},
\end{align*}
holds for all $\xvec,\yvec\in\X\times\Y$. Note that this latter assumption is satisfied whenever the operator norm of 
each
$\wh{\Mm}(t)$ is upper bounded by $L_M$ with probability one.

This noise model is often more realistic than simply 
assuming that $\xivec_{x}(t)$ and $\xivec_{y}(t)$ have uniformly bounded norms, and is much more challenging to 
work with: notably, these noise variables scale with the iterates $\xt$ and $\yt$, and may thus grow uncontrollably as 
the iterates grow large. In particular, the noise is precisely of this form in our application to reinforcement 
learning presented in Section~\ref{sec:RL}.


	The final output of the algorithm will be denoted as $(\oxvec_T,\oyvec_T)$, 
	and due to noise in the gradients, 
	its quality will be measured in terms of the \emph{duality gap}, defined 
	with respect to a comparator point 
	$(\xvec^*,\yvec^*)$ as
	\[
	G(\xvec^{*}, \yvec^{*}) = f(\oxvec_T,\yvec^{*}) - 
	f(\xvec^{*},\oyvec_T).
	\]
	We will allow the comparator points $\xvec^*,\yvec^*$ to depend on the 
	entire interaction history $\F_T$, and in 
	particular allow choices such as $(\xvec^{*}, \yvec^{*}) = 
	\argmax_{\xvec,\yvec \in \mathcal{S}} G(\xvec,\yvec)$ for 
	arbitrary bounded sets $\Sw \subseteq \X\times\Y$. The gap function 
	evaluated at this point is often called a 
	\emph{merit function}, which measures progress towards finding an optimal 
	solution to the min-max optimization problem 
	in question. Another typical choice for comparator point is a saddle point 
	of $f$ that satisfies the inequalities
	\begin{equation}\label{eq:saddlepoint}
		f(\xvec^{*},\yvec)\leq 
		f(\xvec^{*},\yvec^{*})\leq f(\xvec,\yvec^{*})
	\end{equation}
	for all $\xvec\in\X, \yvec\in\Y$. In the rest of the paper, we will use 
	$\xvec^*,\yvec^*$ to refer more generally to 
	any potentially random pair of comparator points, and will call such 
	comparator choices \emph{adaptive}.
	We provide an example that makes use of adaptive comparators in 
	Section~\ref{sec:RL} in the context of reinforcement 
	learning. 
\section{Algorithm and main results}\label{sec:reg}
We now present our algorithmic approach and provide its performance guarantees. For didactic purposes, we start with 
the special case of bilinear games and Euclidean geometries, and then later provide an extension to 
sub-bilinear objectives and more general geometries in Section~\ref{sec:sublinear}.

\subsection{Unconstrained bilinear games}
As a gentle start, we first describe our approach for bilinear games as defined in Section~\ref{sec:prelim} where the 
domains are $\X = \real^m$ and $\Y = \real^n$, and distances are measured in terms of the Euclidean distances in the 
respective spaces. For this case, the core idea of our approach is to run 
stochastic gradient descent/ascent to 
compute the iterates of the two players. As discussed before, this procedure may diverge and produce large gradients 
when run on the original objective, unless the iterates are projected to a bounded set. Our key idea is to replace the 
projection set with an appropriately chosen regularization term added to the objective. Precisely, we introduce the 
regularization functions $\HHx(\xvec) = \frac 12 \twonorm{\xvec - \xvec_1}^2$ and $\HHy(\yvec) = \frac 
12 \twonorm{\yvec - \yvec_1}^2$, and define our algorithm via the following recursive updates:
\begin{equation*}\label{eq:update}
\begin{split}
 \xvec_{t+1} &\!=\! \argmin_{\xvec\in\real^m}\!\ev{\iprod{\xvec}{\gtx} + \varrho_x \HHx(\xvec) + \frac{1}{\eta_x} 
\twonorm{\xvec - \xvec_t}^2}\\
 \yvec_{t+1} &\!=\! 
 \argmin_{\yvec\in\real^n}\!\ev{\!-\!\iprod{\yvec}{\gty}\!+\! 
 \varrho_y \HHy(\yvec) + \frac{1}{\eta_y} 
\twonorm{\yvec - \yvec_t}^2}.
\end{split}
\end{equation*}
For each player, the update rules can be recognized as an instance of Composite Objective MIrror Descent (\COMID, 
\citealp{duchi2010composite}), and accordingly we refer to the resulting algorithm as Composite Objective Gradient 
Descent-Ascent (\COGDA). The updates can be written in closed form as 
\begin{equation}\label{eq:update-closedform}
\begin{split}
 \xvec_{t+1} &= \frac{\xvec_t - \eta_x \gtx}{1+\varrho_x \eta_x} + 
 \frac{\varrho_x \eta_x \xvec_1}{1+\varrho_x 
\eta_x}\\
 \yvec_{t+1} &= \frac{\yvec_t + \eta_y \gty}{1+\varrho_y \eta_y} + 
 \frac{\varrho_y \eta_y \yvec_1}{1+\varrho_y 
\eta_y}.
\end{split}
\end{equation}
This expression has a clear intuitive interpretation: for the min-player, it is a convex combination of the standard 
SGD update $\xvec_t - \eta_x \gtx$ and the initial point $\xvec_1$, with 
weights that depend on the regularization 
parameter $\varrho_x$. Setting $\varrho_x = 0$ recovers the standard SGD update and makes the algorithm vulnerable to 
divergence issues. The overall method closely resembles the \emph{stabilized online mirror descent} method of 
\citet{fang2022online}, and we will accordingly refer to the effect of the newly introduced regularization term as 
\emph{stabilization}.

After running the above iterations for $T$ steps, the algorithm outputs $\overline{\xvec}_T = \frac{1}{T} \sum_{t=1}^T 
\xvec_t$ and $\overline{\yvec}_T = \frac{1}{T} \sum_{t=1}^T \yvec_t$.
The following theorem is our main result regarding the performance of this algorithm.
\begin{theorem}\label{thm:res1}
Let $\varrho_{y}=4\eta_{x}L_M^{2}$ and $\varrho_{x}=4\eta_{y}L_M^{2}$. 
Then, the duality gap achieved by \COGDA satisfies the following bound against 
any adaptive comparator $(\xvec^*,\yvec^*)\in\real^m\times\real^n$:\looseness=-1
	\begin{align*}
		\EE{G(\xvec^{*};\yvec^{*})}
		&\leq
		\pa{\frac{1}{\eta_{y}T} + 2\eta_{x}L_M^{2}}
		\EE{\sqtwonorm{\yvec^{*} - \yvec_{1}}}\\
		&\quad
		+ \pa{\frac{1}{\eta_{x}T} + 2\eta_{y}L_M^{2}}
		\EE{\sqtwonorm{\xvec^{*} - \xvec_{1}}}\\
		&\quad
		+\frac{2\eta_{y}}{T}\sum_{t=1}^{T}
		\EE{\sqtwonorm{\wh{\Mm}(t)\transpose\xvec_{1} - \hat{\cvec}(t)}}\\
		&\quad
		+ \frac{2\eta_{x}}{T}\sum_{t=1}^{T}
		\EE{\sqtwonorm{\wh{\Mm}(t)\yvec_{1}  +\hat{\bvec}(t)}}.
	\end{align*}
In particular, setting $\xvec_1 = 0$ and $\yvec_1 = 0$ and $\eta_x = 1/{L_M
\sqrt{2T}}$ and 
$\eta_y = 1/{L_M\sqrt{2T}}$,
the duality gap is upper bounded as
\begin{align*}
		\EE{G(\xvec^{*},\yvec^{*})}
		&\!=\! \OO\pa{\frac{L_M^2\EE{\twonorm{\yvec^*}^2 + \twonorm{\xvec^*}^2} 
		+ 
		L_b^2 + L_c^2}{L_M\sqrt{T}}}.
	\end{align*}
\end{theorem}
We stress that this bound holds against arbitrary data-dependent choices of the 
comparator points 
$(\xvec^*,\yvec^*)$, and in particular also against the choice $(\xvec^*,\yvec^*) = \argmax_{\xvec,\yvec \in 
\mathcal{S}} G(\xvec,\yvec)$ for any bounded set $\mathcal{S}\subset \real^n\times\real^m$, recovering a standard notion 
of merit function studied in the context of saddle-point optimization 
\citep{nemirovski2009robust}.
It is insightful to compare this bound side by side with the one we would get by running primal-dual stochastic 
gradient without regularization. By standard arguments (see, e.g., 
\citealp{Zin03,nemirovski2009robust}), the following bound is easy to prove:
\begin{align*}
 \EE{G(\xvec^{*};\yvec^{*})} \le& \frac{\EE{\sqtwonorm{\xvec^{*} - 
 \xvec_{1}}}}{\eta_x T} + \frac{\EE{\sqtwonorm{\yvec^{*} - 
\yvec_{1}}}}{\eta_y T} 
\\
& + \frac{\eta_{x}}{T}\sum_{t=1}^{T}
		\EE{\sqtwonorm{\wh{\Mm}(t)\yvec_{t}  +\wh{\bvec}(t)}}
\\
& +\frac{\eta_{y}}{T}\sum_{t=1}^{T}		\EE{\sqtwonorm{\wh{\Mm}(t)\transpose\xvec_{t} - \wh{\cvec}(t)}}.
\end{align*}
A major problem with this bound is that it features the squared stochastic gradient norms evaluated at $\xvec_t$ and 
$\yvec_t$, which are generally unbounded, which makes this guarantee void of meaning without projecting the updates. 
Our own guarantee stated above replaces these gradient norms with the norms of the gradients evaluated at \emph{the 
initial point} $\xvec_1,\yvec_1$, which is \emph{always bounded} irrespective of how large the actual iterates 
$\xvec_t,\yvec_t$ get.

At first, it may seem surprising that such an improvement is possible to achieve by such a simple regularization trick. 
To provide some insight about how regularization helps us achieve our goal, we provide the brief proof sketch
of the above statement here (and defer the full proof to Appendix~\ref{appx:res1}).
\begin{proof}[Proof sketch of Theorem~\ref{thm:res1}.]
 Consider $(\xvec^*,\yvec^*)\in\real^m\times\real^n$. As the first step, 
 we 
 introduce the notation 
 $\freg\pa{\xvec,\yvec} = f\pa{\xvec,\yvec} + \frac{\varrho_{x}}{2} 
 \twonorm{\xvec - \xvec_1}^2 - \frac{\varrho_{y}}{2} \twonorm{\yvec - 
 	\yvec_1}^2 $ and rewrite the expected duality gap as 
 \begin{align*}
 	\EE{G(\xvec^{*};\yvec^{*})}
 	&\leq \frac{1}{T}\sum_{t=1}^{T}
 	\EE{\freg\pa{\xvec_{t},\yvec^{*}} - \freg\pa{\xvec^{*},\yvec_{t}}}\\
 	&\quad
 	+\frac{\varrho_x}{2T}\sum_{t=1}^{T}\EE{\sqtwonorm{\xvec^{*} - \xvec_{1}} - 
 	\sqtwonorm{\xvec_{t} - \xvec_{1}}}\\
 	&\quad
 	 	+\frac{\varrho_y}{2T}\sum_{t=1}^{T}\EE{\sqtwonorm{\yvec^{*} - \yvec_{1}} - 
 	\sqtwonorm{\yvec_{t} - \yvec_{1}}}.
 \end{align*}
 The first term in this decomposition then can be further written as the sum of \emph{regrets} of the min and the max 
players:
 \begin{align*}
 	&\frac{1}{T}\sum_{t=1}^{T}\EE{\freg\pa{\xvec_{t},\yvec^{*}} - 
 	\freg\pa{\xvec^{*},\yvec_{t}}}\\
 	&\qquad=
 	\frac{1}{T}\sum_{t=1}^{T}
 	\EE{\freg\pa{\xvec_{t},\yvec^{*}} - \freg\pa{\xvec_{t},\yvec_{t}}}\\
 	&\qquad\quad
 	+ \frac{1}{T}\sum_{t=1}^{T}
 	\EE{\freg\pa{\xvec_{t},\yvec_{t}}- \freg\pa{\xvec^{*},\yvec_{t}}}
 \end{align*}
These terms can then be controlled via the standard regret analysis of \COMID 
due to \citet{duchi2010composite}, and an additional ``ghost-iterate'' trick to 
account for adaptive comparators (see Section~3 of 
\citealp{nemirovski2009robust}). In 
particular, a few lines of calculations (along the lines of the online gradient descent analysis of 
\citealp{Zin03}) yield the following bound on the sum of the two regrets:
\begin{align*}
 	&\frac 1T \sum_{t=1}^{T}\EE{\freg\pa{\xvec_{t},\yvec^{*}} - 
 	\freg\pa{\xvec^{*},\yvec_{t}}}\\
 	&\qquad\leq \frac{\EE{\sqtwonorm{\yvec^{*} - \yvec_{1}}}}{\eta_{y}T}
 	+\frac{\eta_{y}}{T}\sum_{t=1}^{T}\EE{\sqtwonorm{\gty}}\\
 	&\qquad\quad+ 
 	\frac{\EE{\sqtwonorm{\xvec^{*} - \xvec_{1}}}}{\eta_{x}T}
 	+ \frac{\eta_{x}}{T}\sum_{t=1}^{T}\EE{\sqtwonorm{\gtx}}
\end{align*}
Recalling the form of the gradient estimators, we note that 
\begin{align*}
 &\EE{\sqtwonorm{\gtx}} = \EE{\sqtwonorm{\wh{\Mm}(t)\yvec_{t}  +\hat{\bvec}(t)}}
 \\
 &\,\le 2\EE{\sqtwonorm{\wh{\Mm}(t)\pa{\yvec_{t}  - \yvec_1}}}
+ 2\EE{\sqtwonorm{\wh{\Mm}(t)\yvec_{1}  +\hat{\bvec}(t)}}
  \\
  &\,\le 2L_M^2 \EE{\sqtwonorm{\yvec_{t}  - \yvec_1}}
  + 2\EE{\sqtwonorm{\wh{\Mm}(t)\yvec_{1}  +\hat{\bvec}(t)}},
\end{align*}
and a similar bound can be shown for the norm of $\gty$ as well.
Putting the above bounds together and setting $\varrho_y \ge 4L_M \eta_x$ and 
$\varrho_x \ge 4L_M \eta_y$ gives the 
result.
\end{proof}

As can be seen from the proof sketch above, the role of the additional regularization term for the $x$-player is to 
eliminate the gradient norms appearing in the regret bound of the $y$-player. This effect kicks in once the 
regularization parameter $\varrho_x$ becomes large enough, so that the corresponding negative term in the regret bound 
of the first player can overpower the positive term appearing on the bound of the opposite player. The same story 
applies to the second player. Note that while the regularization pulls the iterates closer to the initial point 
$\xvec_1,\yvec_1$, it does not explicitly guarantee that they remain uniformly bounded at all times $t$, and in fact 
such claim seems impossible to show in general due to the noise in the gradient estimates. Remarkably, the analysis 
above works seamlessly for noisy gradient estimates, even though the gradient noise can grow proportionally with the 
size of the iterates. Another technical challenge that the analysis needs to 
address is the potential 
data-dependence of the comparators $(\xvec^*,\yvec^*)$, which introduces a potential bias into the estimates of the gap 
function $G(\xvec^*,\yvec^*)$. We handle this bias by adapting an elegant ``ghost-iterate'' trick of 
\citet{nemirovski2009robust}, which we regard as an application of the technique of \citet{rakhlin2017equivalence} for 
controlling suprema of a collection of martingales via online learning.

\subsection{Sub-bilinear games and general divergences}\label{sec:sublinear}
After setting the stage in the previous section, we are now ready to introduce our method in its full generality. 
Specifically, we are going to consider a somewhat broader class of objective functions, and provide mirror-descent 
style performance guarantees that measure distances in terms of Bregman divergences. We are going to take inspiration 
from Theorem~\ref{thm:res1} and its proof we have just presented: in short, the idea is to add appropriate 
regularization terms to the objective that will cancel some otherwise large positive terms in the regret analyses of the 
two players. The choice of the regularization terms will be somewhat more involved in this case, and will require 
taking the structure of the objective function into account.

We will let $\omega_x: \X\ra\real$ and $\omega_y:\Y \ra \real$ be two convex functions, to be called 
the \emph{distance-generating functions} over $\X$ and $\Y$. We suppose that $\omega_x$ is $\gamma_x$-strongly convex 
with respect to the norm $\norm{\cdot}_x$ and similarly that $\omega_y$ is $\gamma_y$-strongly convex 
with respect to $\norm{\cdot}_y$. We will respectively denote the Bregman divergences induced by $\omega_x$ and 
$\omega_y$ as $\DDx{\cdot}{\cdot}$ and $\DDy{\cdot}{\cdot}$. We will assume 
that the objective function satisfies 
the following condition:
\begin{definition}\label{def:sub-bilinear}(\textbf{sub-bilinear function})
	A convex-concave function $f : \X\times \Y \rightarrow \Rn$ is said to be 
	$l$-\emph{sub-bilinear} for some $l > 0$ with respect to the norms 
	$\norm{\cdot}_x$ and $\norm{\cdot}_y$, if its 
subgradients $\gvec_x \in \partial_x f(\xvec,\yvec)$ and $\gvec_y \in \partial_y f(\xvec,\yvec)$ satisfy the following 
conditions for all $\xvec,\yvec$:
\begin{align*}
 \norm{\gvec_x}_{x,*}^2 &\le l^2\pa{\norm{\yvec}_{x,*}^2 + 1},
 \\
 \norm{\gvec_y}_{y,*}^2 &\le l^2\pa{\norm{\xvec}_{y,*}^2 + 1}.
\end{align*}
\end{definition}
This condition effectively states that, for a fixed $\yvec$ (resp.~$\xvec$), the objective function 
$f(\xvec,\yvec)$ is Lipschitz with respect to $\xvec$ (resp.~$\yvec$) with a constant that grows at most as fast as 
$\norm{\yvec}_{x,*}$ (resp.~$\norm{\xvec}_{y,*}$). Put differently, it means that $f$ behaves like a bilinear function
asymptotically as one approaches infinity in each direction, which justifies the name ``sub-bilinear'' (mirroring the 
notion of ``sublinearity'' or ``subadditivity'' in convex analysis, cf.~\citealp{HL01}, Section~C.1). We 
will further suppose that the stochastic gradients themselves satisfy the 
following conditions for some $L > 0$:
\begin{equation}\label{eq:noise_condition}
\begin{split}
 \EEt{\norm{\gtx}_{x,*}}^2 &\le L^2\pa{\norm{\yvec_{t} - \yvec_1}_{x,*}^2 + 1},
 \\
 \EEt{\norm{\gty}_{y,*}}^2 &\le L^2\pa{\norm{\xvec_{t} - \xvec_1}_{y,*}^2 + 1}.
\end{split}
\end{equation}
Supposing that the condition holds with norms respectively centered at $\xvec_1$ and $\yvec_1$ is without loss of 
generality, and in particular one can always verify 
$l^2\bpa{\norm{\xvec}_{y,*}^2 + 1} \le L^2\bpa{\norm{\xvec - 
\xvec_1}_{y,*}^2 + 1}$ at the price of replacing $l$ by a larger factor $L$ 
that may depend on $\norm{\xvec_1}_{y,*}$.

For this setting, our algorithm is an adaptation of \emph{composite-objective mirror descent} (\COMID, 
\citealp{duchi2010composite}), which itself is an adaptation of the classic mirror descent method of 
\citet{nemirovskij1983problem,beck2003mirror}, variants of which have been used broadly since the early days of 
numerical optimization \citep{R76,Martinet1970,Martinet1978}. In particular, we introduce the additional regularization 
functions $\HHx: \X\ra \real$ and $\HHy: \Y\ra \real$ defined respectively for each $\xvec$ and $\yvec$ as 
$\HHx(\xvec) = {\frac{1}{2}}\norm{\xvec - \xvec_1}^2_{y,*}$ and 
$\HHy(\yvec) = {\frac{1}{2}}\norm{\yvec - \yvec_1}^2_{x,*}$, and use these 
as 
additional regularization terms to calculate the following sequence of updates in each round $t=1,2,\dots,T$:
\begin{align*}
 \xvec_{t+1}
		&\!=\! \argmin_{\xvec\in\X}\ev{\iprod{\xvec}{\gtx}
			+ \varrho_x \HHx(\xvec) + \frac{1}{\eta_{x}}\DDx{\xvec}{\xvec_{t}}}
			\\
\yvec_{t+1}
		&\!=\! \argmax_{\yvec\in\Y}\!\ev{\iprod{\yvec}{\gty}
			- \varrho_y \HHy(\yvec) - \frac{1}{\eta_{y}}\DDy{\yvec}{\yvec_{t}}},
\end{align*}
We refer to this algorithm as Composite-Objective Mirror Descent-Ascent (\COMIDA), and provide our main result 
regarding its performance.

\begin{theorem}\label{thm:res2}
	Suppose that $f$ is sub-bilinear and the stochastic gradients satisfy the 
	conditions in 
	Equation~\eqref{eq:noise_condition}. Letting $\varrho_x = \frac{2\eta_y 
		L^{2}}{\gamma_y}$ and 
	$\varrho_y = \frac{2\eta_x L^{2}}{\gamma_x}$, and 
	$\xvec^*,\yvec^*$ be arbitrary adaptive comparator points, the expected 
	duality gap 
	of \COMIDA satisfies the following bound:
	\begin{align*}
		\EE{G(\xvec^{*};\yvec^{*})}
		&\leq
		\frac{2\EE{\DDy{\yvec^{*}}{\yvec_{1}}}}{\eta_y T}
		+ \frac{\varrho_y}{2}\EE{\norm{\yvec^{*} - \yvec_1}^2_{x,*}}
		\\
		&\quad
		+ \frac{2\EE{\DDx{\xvec^{*}}{\xvec_{1}}}}{\eta_x T}
		+ \frac{\varrho_x}{2}\EE{\norm{\xvec^{*} - \xvec_1}^2_{y,*}}\\
		&\quad
		+ L^2\pa{\frac{\eta_y}{\gamma_y } + \frac{\eta_x}{\gamma_x}}.
	\end{align*}
\end{theorem}
The proof is provided in Appendix~\ref{appx:res2}.
Notably, as in the case of Theorem~\ref{thm:res1}, the statement remains valid for adaptively chosen comparator
points such as $(\xvec^{*}, \yvec^{*}) = 
\argmax_{\xvec,\yvec \in \mathcal{S}} G(\xvec,\yvec)$ for arbitrary bounded 
sets $\Sw\subseteq \X\times\Y$.
The most important special case of our setting is when the norms appearing in the statement are dual to each 
other, and in particular $\norm{\cdot}_x = \norm{\cdot}_{y,*}$ and $\norm{\cdot}_y = \norm{\cdot}_{x,*}$, so that 
$\omega_x$ is strongly convex with respect to $\norm{\cdot}_{y,*}$ and $\omega_y$ is strongly 
convex with respect to $\norm{\cdot}_{x,*}$. This is the case for instance when $\X = \Y = \real^m$, $\omega_x = \frac 
12 \norm{\cdot}_{\bm{A}}^2$ and $\omega_y = \frac 12 \norm{\cdot}_{\bm{A}^{-1}}^2$ for a symmetric positive definite 
matrix $\bm{A}\in\real^{m\times m}$. We state a specialized version of our 
statement to this setting below.
\begin{corollary}\label{cor:res2}
	Suppose that $f$ is sub-bilinear and the stochastic gradients satisfy the 
	conditions in 	Equation~\eqref{eq:noise_condition}, and suppose additionally that 
	$\omega_x$ is 	$\gamma_x$-strongly convex with respect to $\norm{\cdot}_{y,*}$ and 
	$\omega_y$ is $\gamma_y$-strongly convex with 
	respect to $\norm{\cdot}_{x,*}$. Set the parameters as  $\varrho_x = 
	\frac{2\eta_y L^{2}}{\gamma_y}$, 
$\varrho_y = \frac{2\eta_x L^{2}}{\gamma_x}$, $\eta_x = \eta_y = 
\sqrt{\gamma_x\gamma_y/L^2T}$. Then, letting 
$\xvec^*,\yvec^*$ be arbitrary adaptive comparator points, the duality gap of 
\COMIDA satisfies the following bound:
	\begin{align*}
		&\EE{G(\xvec^{*};\yvec^{*})}\\
		&\qquad= \OO\pa{\frac{L\pa{\EE{\DDx{\xvec^{*}}{\xvec_{1}} + 
		\DDy{\yvec^{*}}{\yvec_{1}}} 
				+ 1}}{\sqrt{\gamma_x \gamma_y T}}}.
	\end{align*}
\end{corollary}
The proof simply follows from using the definition of strong convexity to upper bound $\gamma_y \norm{\yvec^* - 
\yvec_1}_{x,*}^2 \le 2\DDy{\yvec^*}{\yvec_1}$ and $\gamma_x \norm{\xvec^* - \xvec_1}_{y,*}^2 \le 
2\DDy{\xvec^*}{\xvec_1}$.

The above results enjoy the same initialization-dependent property as the ones we have established earlier for bilinear 
games, with the upgrade that the result now holds in terms of general Bregman divergences and also slightly relaxes the 
conditions on the objective function.

\section{Application to Average Reward Markov Decision 
Processes}\label{sec:RL}
In this section, we apply techniques from the previous section for computing near-optimal 
policies in average-reward Markov Decision Processes (AMDPs). As it is 
well-known, this task can be formulated as a linear program (LP), which in turn can be solved by finding a saddle point 
of the associated Lagrangian. Below, we will only describe the saddle-point optimization problem itself and give more 
context on the problem in Appendix~\ref{app:MDP}. For a full technical description of the LP formulation of optimal 
control in MDPs, we refer to Section~8.8 in the classic textbook of \citet{Puterman1994}.

We consider infinite-horizon AMDPs denoted as $(\Sw,\A,r,P)$ where $\Sw$ is a 
finite state space of cardinality 
$S$, $\A$ is a finite action space of cardinality $A$, $r:\Sw\times\A\rightarrow [0,1]$ a reward function and 
$P:\Sw\times\A\rightarrow\Delta_{S}$ a stochastic transition model. For ease of notation, we often refer to the reward 
vector $\rvec\in\Rn^{\Atot}$ with entries $\{r(s,a)\}_{(s,a)\in\AAtot}$, and the transition matrix 
$\Pm\in\Rn^{\Atot\times\Sw}$ with rows $\Pm_{\pa{s,a},\cdot} = P(\cdot|s,a)\in\Delta_{S}$ for 
$(s,a)\in\AAtot$. We also define the matrix  $\Em\in\real^{SA\times S}$ with 
entries $\Em_{(s,a),s'} = \II{s=s'}$.

The agent-environment interaction in this MDP setting is described 
thus: for $k=1, 2, \cdots, K$ steps, having observed the current state $s_{k}$ 
of the 
environment, the agent takes action $a_{k}$ 
according to some stochastic policy $\pi(\cdot|s_{k})$. In consequence of this 
action, the agent receives an immediate reward $r_{k} = r(s_{k},a_{k})$, and moves to the next 
state $s_{k+1}\sim P(\cdot|s_{k},a_{k})$, from where the interaction continues. The performance of the policy $\pi$ is 
measured in terms of the \emph{long-term average reward} (or \emph{gain}) 
$\rho^{\pi} = \limsup_{K\rightarrow 
	\infty}\frac{1}{K}\EEpi{\sum_{k=1}^{K}r(s_{k},a_{k})}$. The goal of the 
	optimal control problem is to find an 
optimal policy $\pi^{*}$ that achieves maximal average reward: $\pi^{*} =\argmax_{\pi}\rho(\pi)$. We provide more 
details on the existence conditions of such optimal policies in the Appendix.

The Lagrangian associated with the optimal control problem is written as
\begin{align*}
	\LL(\muvec; \vvec) &= \iprod{\muvec}{\rvec} + \iprod{\vvec}{\Pm\transpose\muvec - \Em\transpose\muvec}.
\end{align*}
Here, the primal variable $\muvec\in\Delta_{SA}$ is a probability distribution 
on the state-action space that we will refer to as an 
\emph{occupancy measure} and the dual $\vvec \in \real^S$ is a real-valued 
function that we will refer to as a \emph{value 
function}. The saddle point $(\muvec^*,\vvec^*)$ corresponds to the pair of the 
optimal occupancy measure $\muvec^*$ and 
the optimal value function $\vvec^*$. In most problems of practical interest, the scale of the value functions is 
unknown a priori, and consequently there is no tractable way of coming up with a bounded set $\VV \subset \real^S$ 
that will include the optimal value function $\vvec^*$. Without such prior knowledge, one has to solve the 
\emph{unconstrained} saddle-point optimization problem $\min_{\vvec\in\Rn^{S}}\max_{\muvec\in\Delta_{\Atot}} 	
\LL(\muvec; \vvec)$ in order to find the optimal policy---which is precisely 
the subject of our paper.

We will employ a version of our stochastic primal-dual algorithm to solve the above unconstrained problem. We work in 
the well-studied setting of \emph{planning with random-access models}, where we are given a \emph{simulator} (or 
\emph{generative model}) of the transition function $P$ that we can query at any state-action pair $(s,a)$ for an 
i.i.d.~sample from $P(\cdot|s,a)$. We will use this simulator to build estimators of the gradients
\begin{align*}
	\nabla_{\vvec}\LL(\muvec; \vvec)
	&= \Pm\transpose\muvec - \Em\transpose\muvec\\
	\nabla_{\muvec}\LL(\muvec; \vvec)
	&= \rvec + \Pm\vvec - \Em\vvec,
\end{align*}
with their stochastic estimators calculated for each $t$ as
\begin{align*}
	\gtv &= \vec{e}_{s'_{t}} - \vec{e}_{s_{t}}\\
	\gtmu &= \sum_{(s,a)\in\AAtot}[r(s,a) + 
	v_{t}(\overline{s}'_{t}) 
	- 
	v_{t}(s)]\vec{e}_{(s,a)},
\end{align*}
using i.i.d.~samples $(s_{t},a_{t})\sim\muvec_{t}, s'_{t}\sim P(\cdot|s_{t},a_{t})$, also $\overline{s}'_{t}(s,a)\sim 
P(\cdot|s,a)$ for all $(s,a)\in\AAtot$. This makes for a total of $SA + 1$ queries per gradient computation.

Since in our setting only $\vvec$ is unconstrained, it will be enough to introduce the stabilizing regularization for 
these parameters. With that, our algorithm will initialize $\vvec_1 = 0$ and $\muvec_1$ arbitrarily, and then perform 
the following sequence of updates for all $t=1,2,\dots,T$:
\begin{align*}
		\vvec_{t+1}
		\!&=\!\argmin_{\vvec\in\Rn^{S}}\!\ev{\!\iprod{\vvec}{\gtv}
			\!+\! 
		\frac{1}{2\eta_{v}} \twonorm{\vvec - \vvec_t}^2 \!+\! 
\varrho_v \infnorm{\vvec}^2\!},\\
\muvec_{t+1}
		&= \argmin_{\muvec\in\Delta_{\Atot}}\ev{-\iprod{\muvec}{\gtmu}
			+ \frac{1}{\eta_{\mu}}\DDf{\text{KL}}{\muvec}{\muvec_{t}}},
\end{align*}
where $\DDKL{\muvec}{\muvec'} = \sum_{s,a} \muvec(s,a) \log \frac{\muvec(s,a)}{\muvec'(s,a)}$ is the relative entropy 
(or Kullback--Leibler divergence) between $\muvec$ and $\muvec'$. We refer to 
the resulting algorithm as \COMIDAMDP.

The output of \COMIDAMDP is a policy $\overline{\pi}_T:\Sw\ra\Delta_{\A}$, defined by first computing the 
average of the primal iterates $\overline{\mu}_T = \frac 1T \sum_{t=1}^T \mu_t$, and then setting
\[
 \overline{\pi}_T(a|s) = 
 \frac{\overline{\mu}_T(s,a)}{\sum_{a'\in\A}\overline{\mu}_T(s,a')}
\]
for all $s,a$. Then, adapting a result from \citet{cheng2020reduction}, we can relate the suboptimality of the output 
policy to the duality gap evaluated at a well-chosen pair of comparator points 
$(\muvec^*,\vvec^{\overline{\pi}_T})$:
\[
 \rho^{\pi^*} - \rho^{\overline{\pi}_T} = 
 G(\muvec^{\pi^{*}};\vvec^{\overline{\pi}_{T}}).
\]
Notably, the size of the comparator point $\vvec^{\overline{\pi}_T}$ is unknown a priori, and additionally it depends 
on the interaction history which will necessitate some extra care in our analysis. We once again refer to 
Appendix~\ref{app:MDP} for more details regarding the choice of $\vvec^{\overline{\pi}_T}$ and the formal proof of the 
above claim. 

Our main result in this section is the following.
\begin{theorem}\label{thm:res3}
	Let $\varrho_{v}=4\eta_{\mu}$. Then, the output of \COMIDAMDP satisfies the 
	following bound:
	\begin{align*}
		&\EE{\iprod{\muvec^{\pi^{*}} - \muvec^{\overline{\pi}_{T}}}{\rvec}} 
		\leq \frac{\DDKL{\muvec^{\pi^{*}}}{\muvec_{1}}}{\eta_{\mu}T} + 
		\eta_\mu + 2\eta_{v}
		\\
		&\qquad\qquad\qquad\qquad\qquad
		+
		\pa{\frac{1}{\eta_{v}T} + 4\eta_{\mu}}
		\EE{\sqtwonorm{\vvec^{\overline{\pi}_{T}}}}.
	\end{align*}
	In particular, if the output policy satisfies 
	$\infnorm{\vvec^{\overline{\pi}_{T}}}\le B$ for some $B>0$ and 
	$\mu_1$ is the uniform distribution over $SA$, and tuning the parameters as 
	$\eta_\mu = \sqrt{\frac{\log\pa{SA}}{{S}T}}$ 
	and $\eta_v = \sqrt{SA/T}$, the bound becomes
	\[
	 \EE{\iprod{\muvec^{\pi^{*}} - \muvec^{\overline{\pi}_{T}}}{\rvec}}  = 
	 \OO\pa{\sqrt{\frac{B^4SA \log\pa{SA}}{T}}}.
	\]
\end{theorem}
Thus, the iteration complexity of \COMIDAMDP for finding an $\varepsilon$-optimal policy is of the order $\frac{B^4 SA 
\log(SA)}{\varepsilon^2}$.
We stress, unlike similar prior results such as the ones of 
\citet{MWang2017,jin2020efficiently,cheng2020reduction} that this result 
does not require prior knowledge of $B$.
As each iteration uses $SA + 1$ queries to the generative model, this makes for a total 
of $\frac{B^4 S^2A^2 \log(SA)}{\varepsilon^2}$ query complexity, which is suboptimal in terms of its dependence on 
$SA$, but optimal in terms of $\varepsilon$.

\section{Discussion}\label{sec:conc}
Our work contributes to the rich literature on saddle-point optimization via incremental first-order methods, a subject 
studied at least since the works of \citet{Martinet1970,Martinet1978,R76,nemirovskij1983problem}. In the last few 
years, this topic has enjoyed a massive comeback within the context of optimization for machine learning models, and in 
particular generative adversarial networks (GANs, \citealp{goodfellow2014generative}). The instability of standard 
gradient descent/ascent methods has been pointed out early on during this revival, which brought significant attention 
to a family of methods known extragradient methods, first proposed by \citet{korpelevich1976extragradient} and further 
developed by 
\citet{popov1980modification,nemirovski2004prox,juditsky2011solving,rakhlin2013online,rakhlin2013optimization}. A 
wealth of recent works have contributed to a better understanding of these methods, and most notably established 
last-iterate convergence of extragradient-type methods for a variety of problem settings 
\citep{daskalakis2017training,gidel2018variational,mertikopoulos2018optimistic,mishchenko2020revisiting}. 
The majority of these works assume access to either deterministic gradients or gradients with uniformly 
bounded noise and bounded domain. The assumption of 
bounded noise was more recently lifted in the works of \citet{loizou2021stochastic} and \citet{sadiev2023high}, but 
their assumptions on the noise and the objective function are ultimately 
incompatible with our setting.

We leave several interesting questions open for future work. The biggest of these questions is if the scaling with the 
initialization error $\twonorm{\xvec^* - \xvec_1}^2 + \twonorm{\yvec^* - \yvec_1}^2$ can be improved to 
$\twonorm{\xvec^* - \xvec_1} + \twonorm{\yvec^* - \yvec_1}$. This is obviously possible when we have prior knowledge of 
these norms, and can tune the learning rate to fully optimize the first set of bounds in Theorem~\ref{thm:res1}. 
Without prior knowledge, it is less clear if such improvement is possible, unlike in the case of convex minimization 
problems where there exist efficient algorithms that achieve such improved rates, at least up to log factors 
\citep{streeter2012no,orabona2013dimension,orabona2014simultaneous}. While we 
were not aware of this at the 
time of preparing the first version of this article, the recent work of \citet{jacobsen2023unconstrained} has already 
provided some results that hint at a negative answer: their Theorem 2.3 shows that there exists an online linear 
optimization problem with sub-quadratic gradient growth where quadratic scaling with the comparator norm is 
unavoidable. Whether or not their counterexample can be adapted to our setting 
remains to be seen.\looseness=-1

%

We close by highlighting (one more time) some similarities between our approach and some previously proposed 
methods. We first mention the work of \citet{jacobsen2023unconstrained} that we learned about after completing the 
first version of the present manuscript. Their methods have the advantage of being 
completely parameter-free, and demonstrating a slightly more refined dependence on the comparator norms than our 
guarantees do. On the other hand, our algorithm is arguably much simpler and is thus 
much easier to adapt to more general settings, as evidenced by our main results that are stated in terms of general 
Bregman divergences. In contrast, the analysis of \citet{jacobsen2023unconstrained} is strictly tied to Euclidean norms 
and it is unclear if a generalization to other geometries is straightforwardly possible. Similarly, we believe that 
adjusting their analysis to handle stochastic gradients and data-dependent comparators may not be entirely 
straightforward. Besides this work, our method also bears close similarity to the stabilized online 
mirror descent method of \citet{fang2022online}: their approach introduces a similar regularization term to address 
issues faced by OMD in unconstrained convex minimization problems. Their use of regularization had the purpose of 
allowing time-dependent (and more generally, adaptive) learning-rate schedules, which is ultimately quite different 
from the purpose that we employed this technique for, and also requires a different tuning rule than our method. The 
extension of this technique by \citet{hsieh2021adaptive} to a two-player game setting similar to ours remained 
restricted to consider noiseless gradients and bounded decision sets. Additionally, an anonymous reviewer has 
pointed our attention to the similarity between our approach and an ``anchoring'' technique extensively 
studied under the name of \emph{Halpern iteration} in the optimization literature \citep{halpern1967fixed, 
lieder2021convergence}. All these connections suggest that the simple and natural regularization trick we made use of 
in this paper is a tool of fundamental importance with a large range of diverse uses. 
In light of our results and these observations, we are particularly curious to see if this trick will find 
further uses in the context of saddle-point optimization and 
game theory in the future.

\section*{Acknowledgements}
The authors wish to thank the exceptionally dedicated anonymous reviewer that 
pointed out and helped us fix an 
issue with our definition of the duality gap in the original version of the 
paper.
This project has received funding from the European Research Council (ERC) 
under the European Union’s Horizon 2020 research and innovation programme 
(Grant agreement No.~950180).
%
%


\bibliography{references.bib}
\bibliographystyle{icml2023}

\newpage
\appendix
\onecolumn
\section{Proof of results in \cref{sec:reg}}
In this section, we provide a detailed proof of claims, lemmas and theorems in 
\cref{sec:reg} in the main text.

\subsection{Complete proof of \cref{thm:res1}}\label{appx:res1}
We start by rewriting the expected duality gap evaluated at $(\xvec^{*};\yvec^{*})$ as follows:
\begin{align}
	\EE{G(\xvec^{*};\yvec^{*})}
	\nonumber&= {\EE{f(\oxvec_T,\yvec^{*}) - f(\xvec^{*},\oyvec_T)}}\\
	\nonumber&= {\frac{1}{T}\sum_{t=1}^{T}\EE{f(\xvec_t,\yvec^{*}) - 
	f(\xvec^{*},\yvec_t)}}\\
	\nonumber&= \frac{1}{T}\sum_{t=1}^{T}
	\EE{\freg\pa{\xvec_{t},\yvec^{*}} - \freg\pa{\xvec^{*},\yvec_{t}}}\\
	\nonumber&\quad
	+\frac{\varrho_y}{2T}\sum_{t=1}^{T}\EE{\sqtwonorm{\yvec^{*} - \yvec_{1}} - 
		\sqtwonorm{\yvec_{t} - \yvec_{1}}}\\
	\label{eq:temp1}&\quad
	+\frac{\varrho_x}{2T}\sum_{t=1}^{T}\EE{\sqtwonorm{\xvec^{*} - \xvec_{1}} - 
		\sqtwonorm{\xvec_{t} - \xvec_{1}}}.
\end{align}
To control the first set of terms in the above expression, we apply 
the regret analysis 
in \cref{appx:COMIDA}. 
Precisely, with $\DDx{\xvec}{\xvec'} = \frac{1}{2}\norm{\xvec - \xvec'}^2_{2}$, 
$\DDy{\yvec}{\yvec'} = \frac{1}{2}\norm{\yvec - \yvec'}^2_{2}$ and, 
$\HHx(\xvec) = \frac{1}{2}\norm{\xvec - \xvec_1}^2_{2}$,
$\HHy(\yvec) = \frac{1}{2}\norm{\yvec - \yvec_1}^2_{2}$, we get:
\begin{align}
	\nonumber&\sum_{t=1}^{T}\EE{\freg\pa{\xvec_{t},\yvec^{*}} - 
	\freg\pa{\xvec^{*},\yvec_{t}}}\\
	\nonumber&\qquad\qquad\qquad=
	\sum_{t=1}^{T}\EE{\freg\pa{\xvec_{t},\yvec^{*}} - 
	\freg\pa{\xvec_{t},\yvec_{t}}}
	+
	\sum_{t=1}^{T}\EE{\freg\pa{\xvec_{t},\yvec_{t}} - 
	\freg\pa{\xvec^{*},\yvec_{t}}}\\
	\label{eq:regCOGDA}&\qquad\qquad\qquad\leq
	\frac{\EE{\sqtwonorm{\yvec^{*} - \yvec_{1}}}}{\eta_{y}}
	+\eta_{y}\sum_{t=1}^{T}\EE{\sqtwonorm{\gty}}+ 
	\frac{\EE{\sqtwonorm{\xvec^{*} - \xvec_{1}}}}{\eta_{x}T}
	+ \eta_{x}\sum_{t=1}^{T}\EE{\sqtwonorm{\gtx}}.
\end{align}
To proceed, we recall the assumptions we made on the gradient estimators on the main text, namely that the inequalities
$\EEt{\bigl\|\wh{\Mm}(t)\yvec\bigr\|_2^2} \leq 	
L_{M}^2\sqtwonorm{\yvec}$ and 
$\EEt{\bigl\|\wh{\Mm}(t)\transpose\xvec\|_2^2} \leq 	
L_{M}^2\sqtwonorm{\xvec}$ hold for all $\xvec,\yvec\in\X\times\Y$. 
Using this condition allows us to bound the gradient norms as
\begin{align*}
	\EEt{\sqtwonorm{\gty}}
	&= \EEt{\sqtwonorm{\wh{\Mm}(t)\transpose
		\xvec_{t} - \hat{\cvec}(t)}}\\
	&= \EEt{\sqtwonorm{\wh{\Mm}(t)\transpose
			\pa{\xvec_{t} - \xvec_{1}} + \wh{\Mm}(t)\transpose\xvec_{1} - 
			\hat{\cvec}(t)}}\\
	&\leq 2\EEt{\sqtwonorm{\wh{\Mm}(t)\transpose
			\pa{\xvec_{t} - \xvec_{1}}}} + 
			2\EEt{\sqtwonorm{\wh{\Mm}(t)\transpose\xvec_{1} - \hat{\cvec}(t)}}\\
	&\leq 2L_{M}^{2}\sqtwonorm{\xvec_{t} - \xvec_{1}} + 
	2\EEt{\sqtwonorm{\wh{\Mm}(t)\transpose\xvec_{1} - \hat{\cvec}(t)}},
\end{align*}
where the third line uses the triangle inequality and Cauchy--Schwarz, and the second 
follows from said assumption. Likewise, we can show
\begin{align*}
	\EEt{\sqtwonorm{\gtx}} 	&\leq 2L_{M}^{2}\sqtwonorm{\yvec_{t}-\yvec_{1}} 
	+2\EEt{\sqtwonorm{\wh{\Mm}(t)\yvec_{1} +\hat{\bvec}(t)}}.
\end{align*}
Therefore, by the tower rule and monotonicity of expectation,
\begin{equation*}
	\EE{\sqtwonorm{\gty}}
	= \EE{\EEt{\sqtwonorm{\gty}}}
	\leq 2L_{M}^{2}\EE{\sqtwonorm{\xvec_{t} - \xvec_{1}}} + 
	2\EE{\sqtwonorm{\wh{\Mm}(t)\transpose\xvec_{1} - \hat{\cvec}(t)}},
\end{equation*}
and
\begin{equation*}
	\EE{\sqtwonorm{\gtx}}
	= \EE{\EEt{\sqtwonorm{\gtx}}}
	\leq 2L_{M}^{2}\EE{\sqtwonorm{\yvec_{t}-\yvec_{1}}} 
	+2\EE{\sqtwonorm{\wh{\Mm}(t)\yvec_{1} +\hat{\bvec}(t)}}.
\end{equation*}
Plugging these derivations into the bound of Equation~\cref{eq:regCOGDA} and then combining the result 
with the bound of Equation~\cref{eq:temp1}, we obtain
\begin{align*}
	\EE{G(\xvec^{*};\yvec^{*})}
	&\leq \pa{\frac{1}{\eta_{y}T} + \frac{\varrho_{y}}{2}}
	\EE{\sqtwonorm{\yvec^{*} - \yvec_{1}}}
	+\frac{2\eta_{y}}{T}\sum_{t=1}^{T}
	\EE{\sqtwonorm{\wh{\Mm}(t)\transpose\xvec_{1} - \hat{\cvec}(t)}}\\
	&\quad
	+ \pa{\frac{1}{\eta_{x}T} + \frac{\varrho_{x}}{2}}
	\EE{\sqtwonorm{\xvec^{*} - \xvec_{1}}}
	+ \frac{2\eta_{x}}{T}\sum_{t=1}^{T}
	\EE{\sqtwonorm{\wh{\Mm}(t)\yvec_{1}  +\hat{\bvec}(t)}}\\
	&\quad 
	+ \frac{1}{T}\sum_{t=1}^{T}\EE{\sqtwonorm{\yvec_{t} - \yvec_{1}}}
	\pa{2\eta_{x}L_{M}^{2} - \frac{\varrho_{y}}{2}}
	+ \frac{1}{T}\sum_{t=1}^{T}\EE{\sqtwonorm{\xvec_{t} - \xvec_{1}}}
	\pa{2\eta_{y}L_{M}^{2} - \frac{\varrho_{x}}{2}}.
\end{align*}
By setting $\varrho_{y} = 
4\eta_{x}L_{M}^{2}$ and $\varrho_{x} = 
4\eta_{y}L_{M}^{2}$, we eliminate the last two terms in the bound above and 
arrive at the result stated in the theorem:
\begin{align*}
	\EE{G(\xvec^{*};\yvec^{*})}
	&\leq \pa{\frac{1}{\eta_{y}T} + 2\eta_{x}L_{M}^{2}}
	\EE{\sqtwonorm{\yvec^{*} - \yvec_{1}}}
	+\frac{2\eta_{y}}{T}\sum_{t=1}^{T}
	\EE{\sqtwonorm{\wh{\Mm}(t)\transpose\xvec_{1} - \hat{\cvec}(t)}}\\
	&\quad
	+ \pa{\frac{1}{\eta_{x}T} + 2\eta_{y}L_{M}^{2}}
	\EE{\sqtwonorm{\xvec^{*} - \xvec_{1}}}
	+ \frac{2\eta_{x}}{T}\sum_{t=1}^{T}
	\EE{\sqtwonorm{\wh{\Mm}(t)\yvec_{1}  +\hat{\bvec}(t)}}.
\end{align*}
\qed

\subsection{Proof of \cref{thm:res2}}\label{appx:res2}
Consider the expected duality gap at arbitrary adaptive comparator points 
$(\xvec^*,\yvec^*)$:
\[
\EE{G(\xvec^{*};\yvec^{*})}
= \EE{f(\oxvec_T,\yvec^{*}) - f(\xvec^{*},\oyvec_T)}.
\]
By the convex-concave property of $f$ and straightforward 
derivations, 
we can rewrite the above gap in terms of the regret of a min-max 
optimization scheme and regularization terms as 
\begin{align}
	\nonumber\EE{G(\xvec^{*};\yvec^{*})}
	&= \EE{f(\oxvec_T,\yvec^{*}) - f(\xvec^{*},\oyvec_T)}\\
	\nonumber&\leq 
	\frac{1}{T}\sum_{t=1}^{T}
	\EE{f\pa{\xvec_{t},\yvec^{*}} - f\pa{\xvec^{*},\yvec_{t}}}\\
	\nonumber&= 
	\frac{1}{T}\sum_{t=1}^{T}
	\EE{f\pa{\xvec_{t},\yvec^{*}} - f\pa{\xvec_{t},\yvec_{t}}}
	+ 
	\frac{1}{T}\sum_{t=1}^{T}
	\EE{f\pa{\xvec_{t},\yvec_{t}}- f\pa{\xvec^{*},\yvec_{t}}}\\
	\nonumber&= 
	\frac{1}{T}\sum_{t=1}^{T}
	\EE{\freg\pa{\xvec_{t},\yvec^{*}} - \freg\pa{\xvec_{t},\yvec_{t}}}
	+ 
	\frac{1}{T}\sum_{t=1}^{T}
	\EE{\freg\pa{\xvec_{t},\yvec_{t}}- \freg\pa{\xvec^{*},\yvec_{t}}}\\
	\label{eq:temp2}&\quad
	+\frac{\varrho_y}{T}\sum_{t=1}^{T}\EE{\HHy(\yvec^{*}) - \HHy(\yvec_{t})}
	+\frac{\varrho_x}{T}\sum_{t=1}^{T}\EE{\HHx(\xvec^{*}) - \HHx(\xvec_{t})},
\end{align}
where in this case,
\[
\freg(\xvec,\yvec) = f(\xvec,\yvec) + \frac{\varrho_x}{2} \HHx(\xvec) - 
\frac{\varrho_y}{2} 
\HHy(\yvec).
\]
The rest of the proof is split in two parts. First, we control regularized 
regret of the min and max players, corresponding to the first two sums appearing on the right-hand side of the above 
bound. Then, substituting the resulting bound back into 
\cref{eq:temp2}, we take advantage of the negative terms $\HHx(\xvec_{t})$ and $\HHy(\yvec_{t})$ appearing on the right 
hand side to cancel out some potentially large terms in the regret analysis, arriving at a bound that is robust to large 
iterates.

\subsubsection{Regret Analysis of \COMIDA on a Regularized Objective}\label{appx:COMIDA}
This part of the proof is based on the regret analysis of Composite 
Mirror Descent (\COMID) for stochastic convex optimization. The proof is a more-or-less standard exercise in convex 
analysis (appearing, e.g., as Theorem~8 of \citet{duchi2010composite}), and we provide it for completeness as 
\cref{lem:COMID} in \cref{appx:aux}. In this section, we 
directly apply the implied guarantee on the regret of \COMID against an adaptive
comparator in \cref{cor:COMID} to control the regret of each player. 

For the max player, we denote the 
loss in round $t$ as $\lt(\yvec) = 
-f(\xvec_{t},\yvec)$ for $\yvec\in\Y$ and we define its regularized loss as $\ltreg(\yvec) = 
-f(\xvec_{t},\yvec) + \varrho_{y}\HHy\pa{\yvec}$. Then, the total expected 
regret of the max player on the regularized objective can be rewritten as
\begin{equation*}
	\sum_{t=1}^{T}
	\EE{\freg\pa{\xvec_{t},\yvec^{*}} - \freg\pa{\xvec_{t},\yvec_{t}}}
	= \sum_{t=1}^{T}
	\EE{\ltreg\pa{\yvec_{t}} - \ltreg\pa{\yvec^{*}}}.
\end{equation*}
Notice that $\ltreg\pa{\cdot}$ is convex by the concave property of 
$f(\xvec_{t},\cdot)$. 
We will bound the regret using \cref{cor:COMID}, with initial iterate $\uvec_{1} = \yvec_{1}$, gradient estimates $\gtu 
= -\gty$ and gradients $\gu = -\gy$. Also, we will set $\U = \Rn^{n}$, 
$\omega_u=\omega_y$, $\eta_{u}=\eta_{y}$ and 
$\varrho_{u}=\varrho_{y}$. With $\yvec^{*}$ adaptive and potentially 
dependent on the interaction history, this gives
\begin{equation*}
	\sum_{t=1}^{T}
	\EE{\freg\pa{\xvec_{t},\yvec^{*}} - \freg\pa{\xvec_{t},\yvec_{t}}}
	\leq \frac{2\EE{\DDy{\yvec^{*}}{\yvec_{1}}}}{\eta_y}
	+ \frac{\eta_y}{\gamma_y}\sum_{t=1}^{T}\EE{\sqopnorm{\gty}{y,*}}
	+ \varrho_y\EE{\HHy(\yvec_{1})}.
\end{equation*}
Likewise, reusing previous notation we denote the loss of the min player as 
in round $t$ as $\lt(\xvec) = f(\xvec,\yvec_{t})$. Since $f(\cdot,\yvec_{t})$ 
is convex, and by equivalence of the minimization 
step of \COMIDA to that of \cref{cor:COMID} when $\uvec_{1} = \xvec_{1}$, 
$\gtu = \gtx$, $\gu = \gx$, 
$\U=\Rn^{m}$, $\omega_u=\omega_x$, $\eta_{u}=\eta_{x}$ and 
$\varrho_{u}=\varrho_{x}$, we can bound the regret of the min player against an 
adaptive comparator $\xvec^{*}$ as follows:
\begin{equation*}
	\sum_{t=1}^{T}\EE{\freg\pa{\xvec_{t},\yvec_{t}} - 
		\freg\pa{\xvec^{*},\yvec_{t}}}
	\leq
	\frac{2\EE{\DDx{\xvec^{*}}{\xvec_{1}}}}{\eta_x}
	+ \frac{\eta_x}{\gamma_x}\sum_{t=1}^{T}\EE{\sqopnorm{\gtx}{x,*}}
	+ \varrho_x\EE{\HHx(\xvec_{1})}.
\end{equation*}
Therefore, the total expected regret of \COMIDA on the regularized objective is 
bounded above as follows:
\begin{align*}
	&\sum_{t=1}^{T}\EE{\freg\pa{\xvec_{t},\yvec^{*}} - 
	\freg\pa{\xvec_{t},\yvec_{t}}}
	+
	\sum_{t=1}^{T}\EE{\freg\pa{\xvec_{t},\yvec_{t}} - 
	\freg\pa{\xvec^{*},\yvec_{t}}}\\
	&\quad\leq
	\frac{2\EE{\DDy{\yvec^{*}}{\yvec_{1}}}}{\eta_y}
	+ \frac{\eta_y}{\gamma_y}\sum_{t=1}^{T}\EE{\sqopnorm{\gty}{y,*}}
	+ \varrho_y\EE{\HHy(\yvec_{1})}
	\\
	&\quad\qquad+
	\frac{2\EE{\DDx{\xvec^{*}}{\xvec_{1}}}}{\eta_x}
	+ \frac{\eta_x}{\gamma_x}\sum_{t=1}^{T}\EE{\sqopnorm{\gtx}{x,*}}
	+ \varrho_x\EE{\HHx(\xvec_{1})}.
\end{align*}
This completes the first part of the proof.

\subsubsection{Eliminating the gradient norms}
For the second part, we make use of our specific definition of the 
regularization function: $\HHx(\xvec) = \frac{1}{2}\norm{\xvec - 
\xvec_1}^2_{y,*}$ and 
$\HHy(\yvec) = \frac{1}{2}\norm{\yvec - \yvec_1}^2_{x,*}$. In this 
case $\HHx(\xvec_{1}) = \HHy(\yvec_{1}) = 0$. Then, plugging in the bounds 
from \cref{appx:COMIDA} in the expected duality gap we have:
\begin{align}
	\EE{G(\xvec^{*};\yvec^{*})}
	\nonumber&\leq
	\frac{2\EE{\DDy{\yvec^{*}}{\yvec_{1}}}}{\eta_y T}
	+ \frac{\eta_y}{\gamma_y T}\sum_{t=1}^{T}\EE{\sqopnorm{\gty}{y,*}}
	+ \frac{2\EE{\DDx{\xvec^{*}}{\xvec_{1}}}}{\eta_x T}
	+ \frac{\eta_x}{\gamma_x T}\sum_{t=1}^{T}\EE{\sqopnorm{\gtx}{x,*}}\\
	\label{eq:gapCOMIDA}&\quad
	+\frac{\varrho_y}{2T}\sum_{t=1}^{T}\EE{\norm{\yvec^{*} - \yvec_1}^2_{x,*} - 
		\norm{\yvec_{t} - \yvec_1}^2_{x,*}}
	+\frac{\varrho_x}{2T}\sum_{t=1}^{T}\EE{\norm{\xvec^{*} - \xvec_1}^2_{y,*} - 
		\norm{\xvec_{t} - \xvec_1}^2_{y,*}}.
\end{align}
To proceed, we make crucial use of our noise condition stated as \cref{eq:noise_condition} in the main text so that we 
can bound the gradient norms as
\begin{equation*}
	\EE{\sqtwonorm{\gty}}
	= \EE{\EEt{\sqtwonorm{\gty}}}
	\leq \EE{L^2\pa{\norm{\xvec_{t} - \xvec_{1}}_{y,*}^2 + 1}}.
\end{equation*}
Also,
\begin{equation*}
	\EE{\sqtwonorm{\gtx}}
	= \EE{\EEt{\sqtwonorm{\gtx}}}
	\leq \EE{L^2\pa{\norm{\yvec_{t} - \yvec_{1}}_{x,*}^2 + 1}}.
\end{equation*}
Plugging these into the bound of \cref{eq:gapCOMIDA} gives
\begin{align*}
	\EE{G(\xvec^{*};\yvec^{*})}
	&\leq
	\frac{2\EE{\DDy{\yvec^{*}}{\yvec_{1}}}}{\eta_y T}
	+ \frac{\eta_y}{\gamma_y }L^2
	+ \frac{\varrho_y}{2}\EE{\norm{\yvec^{*} - \yvec_1}^2_{x,*}}
	\\
	&\quad
	+ \frac{2\EE{\DDx{\xvec^{*}}{\xvec_{1}}}}{\eta_x T}
	+ \frac{\eta_x}{\gamma_x}L^2
	+ \frac{\varrho_x}{2}\EE{\norm{\xvec^{*} - \xvec_1}^2_{y,*}}\\
	&\quad
	+\frac{1}{T}\sum_{t=1}^{T}\EE{\norm{\yvec_{t} - 
	\yvec_1}^2_{x,*}}\pa{\frac{\eta_x L^{2}}{\gamma_x} - 
	\frac{\varrho_{y}}{2}}
	+\frac{1}{T}\sum_{t=1}^{T}\EE{\norm{\xvec_{t} - 
	\xvec_1}^2_{y,*}}\pa{\frac{\eta_y L^{2}}{\gamma_y} - 
	\frac{\varrho_{x}}{2}}.
\end{align*}
Lastly, choosing $\varrho_{y} = \frac{2\eta_x L^{2}}{\gamma_x}$ and 
$\varrho_{x} = \frac{2\eta_y L^{2}}{\gamma_y}$ results in the bound 
stated in the theorem:
\begin{align*}
	\EE{G(\xvec^{*};\yvec^{*})}
	&\leq
	\frac{2\EE{\DDy{\yvec^{*}}{\yvec_{1}}}}{\eta_y T}
	+ \frac{\eta_y L^2}{\gamma_y }
	+ \frac{\varrho_y}{2}\EE{\norm{\yvec^{*} - \yvec_1}^2_{x,*}}
	\\
	&\quad
	+ \frac{2\EE{\DDx{\xvec^{*}}{\xvec_{1}}}}{\eta_x T}
	+ \frac{\eta_x L^2}{\gamma_x}
	+ \frac{\varrho_x}{2}\EE{\norm{\xvec^{*} - \xvec_1}^2_{y,*}}.
\end{align*}
\qed

\newpage
\section{Analysis for the Average-Reward MDP Setting}\label{app:MDP}
\subsection{Problem setup}
First we briefly recall some general concepts related to average-reward MDPs (and refer the reader to Chapter~8 of 
\citealp{Puterman1994} for a more detailed introduction into the topic). Consider infinite-horizon AMDPs denoted 
as $(\Sw,\A,r,P)$ where $\Sw$ is a finite state space of cardinality $S$, 
$\A$ is a finite action space of cardinality $A$, $r:\Sw\times\A\rightarrow 
[0,1]$ a reward model and $P:\Sw\times\A\rightarrow\Delta_{S}$ a stochastic 
transition model.
For ease of notation, we often refer to the reward vector $\rvec\in\Rn^{\Atot}$ 
with $\{r(s,a)\}_{(s,a)\in\AAtot}$ entries, and the transition matrix 
$\Pm\in\Rn^{\Atot\times\Sw}$ with $\Pm[s,a] = p(\cdot|s,a)\in\Delta_{S}$ for 
$(s,a)\in\AAtot$.

In this work, we primarily focus on the class of AMDPs where each policy $\pi$ has a well-defined unique stationary 
state distribution (or state-occupancy measure) $\nu^{\pi}:\Sw\ra[0,1]$, defined for each $s$ as 
\begin{align*}
	\nu^{\pi}(s)
	&=\lim_{K\rightarrow\infty}\frac{1}{K}\sum_{k=1}^K \PP{x_{k} = x\,|\,\pi}.
\end{align*}
The stationary distribution can be seen to satisfy the linear system of equations $\nu^{\pi}(s)
	=\sum_{(s',a')}p(s|s',a')\pi(a'|s')\nu^{\pi}(s')$ for all $s\in\Sw$.
Hence, the corresponding stationary 
state-action distribution (or \emph{state-action occupancy measure}) 
$\mu^{\pi}(s,a) = \pi(a|s)\nu^{\pi}(s)$ for $(s,a)\in\AAtot$ is also unique, and we 
can write the average-reward objective as $\rho^{\pi} = \iprod{\muvec^{\pi}}{\rvec}$. 
This compact representation of the reward criterion and occupancy measure
inspires the linear programming approach to optimal control in MDPs, wherein we are 
interested in solving the linear program
\begin{equation}\label{eq:primal-lp-AMDP}
	\begin{alignedat}{2}
		& \max_{\muvec\in\Rn^{\Atot}}  &\quad& \iprod{\muvec}{\rvec} \\
		& \subjectto && \Em\transpose\muvec =\Pm\transpose\muvec \\
		&&&\iprod{\muvec}{\vec{1}}=  1\\
		&&& \muvec \ge 0.
	\end{alignedat}
\end{equation}
In the above expressions, the operator $\Em: \real^{\Atot}\ra \real^{\Sw}$ is defined as 
$(\Em\transpose\muvec)(s) = \sum_{a}\mu(s,a)$ for $s\in\Sw$. This LP is motivated by the fact that the set of 
distributions $\muvec$ that satisfy the constraints exactly corresponds to the set of stationary state-action 
distributions that can be potentially induced by a stationary policy in the MDP.

We also define the value function (or bias function) of policy $\pi$ as $v^{\pi}:\Sw\ra\real$, taking the following 
value in each state $s\in\Sw$:
\begin{align}
	\nonumber v^{\pi}(s)
	&= \lim_{K\rightarrow\infty} \EEcpi{\sum_{k=1}^{K}\pa{r(s_{k},a_{k}) 
		- \rho^{\pi}}}{s_{0} = s}\\
	\label{eq:valuefn}&= \sum_{a}\pi(a|s)\left[r(s,a) - \rho^{\pi} + 
	\iprod{p(\cdot|s,a)}{\vvec^{\pi}}\right].
\end{align}
Then, the value function of an optimal policy maximizing $\rho^{\pi}$ can be shown to be an optimal solution of the 
dual of the LP~\eqref{eq:primal-lp-AMDP}, written as follows:
\begin{equation}\label{eq:dual-lp-AMDP}
	\begin{alignedat}{2}
		& \min_{\rho \in \Rn, \vvec\in\VV}  &\quad& \rho \\
		& \subjectto && \Em \vvec \geq \rvec + \Pm\vvec - \vec{1}\rho.
	\end{alignedat}
\end{equation}

Finding an optimal solution to either of the LPs can be equivalently phrased as solving the following bilinear game:
\begin{equation}\label{eq:AMDP}
	\min_{\vvec\in\VV}\max_{\muvec\in\Delta_{\Atot}} 
	\LL(\vvec; \muvec),
\end{equation}
with the Lagrangian associated with the LPs is defined as
\begin{align*}
	\LL(\vvec; \muvec)
	&= \iprod{\muvec}{\rvec} + 
	\iprod{\vvec}{\Pm\transpose\muvec - 
		\Em\transpose\muvec} + \rho(1-\iprod{\muvec}{\vec{1}})\\
	&= \iprod{\muvec}{\rvec} + 
	\iprod{\vvec}{\Pm\transpose\muvec - 
		\Em\transpose\muvec}.
\end{align*}
The gradients of the above objective are respectively expressed as 
\begin{align*}
	\nabla_{\vvec}\LL(\vvec; \muvec)
	= \Pm\transpose\muvec - \Em\transpose\muvec \qquad\mbox{and}\qquad
	\nabla_{\muvec}\LL(\vvec; \muvec)
	= \rvec + \Pm\vvec - \Em\vvec.
\end{align*}
Now in the context of planning, it is assumed that the transition model is 
unknown, hence the gradients cannot be computed exactly. Rather, we assume 
access to an accurate simulator which can be queried at any state-action pair 
$(s,a)\in\AAtot$ to obtain a sample next state $s'\sim p(\cdot|s,a)$. Indeed, 
with $\vvec_{t},  \muvec_{t}$ determined by the end of round $t-1$, we can 
compute unbiased estimates in round $t$ as:
\begin{align*}
	\gtv &= \vec{e}_{s'_{t}} - \vec{e}_{s_{t}}\\
	\gtmu &= \sum_{(s,a)\in\AAtot}\bpa{r(s,a) + 
	v_{t}(\overline{s}'_{t}(s,a)) 
	- 
	v_{t}(s)}\vec{e}_{(s,a)},
\end{align*}
using i.i.d samples $(s_{t},a_{t})\sim\muvec_{t}, s'_{t}\sim 
p(\cdot|s_{t},a_{t})$, also $\overline{s}'_{t}(s,a)\sim p(\cdot|s,a)$ for all 
$(s,a)\in\AAtot$. 

Our aim is to find a near-optimal policy with a polynomial number of queries to the generative model, by running a 
version of gradient descent-ascent on the Lagrangian $\LL$. In particular, we aim to derive a bound on the 
suboptimality of the output policy in terms of the optimization-error guarantee that we obtain by running our algorithm.
To achieve this, a key quantity to study is the expected gap of the averaged iterates
$(\overline{\muvec}_T,\overline{\vvec}_T)\in\Delta_{\Atot}\times\VV$ 
against arbitrary comparators 
$(\muvec^{*},\vvec^{*})\in\Delta_{\Atot}\times\VV$ 
denoted 
as
\begin{align}\label{eq:AMDP_gap}
	\EE{G(\muvec^{*};\vvec^{*})}
	&=
	\EE{\LL(\muvec^{*}; \overline{\vvec}_T) - \LL(\overline{\muvec}_T; 
		\vvec^{*})},
\end{align}
where $(\overline{\muvec}_T,\overline{\vvec}_T) = 
\pa{\frac{1}{T}\sum_{t=1}^{T}\muvec_{t},\frac{1}{T}\sum_{t=1}^{T}\vvec_{t}}$ 
and $\overline{\pi}_{T}$ are as described in the main text. Then, a relationship between the duality gap and the policy 
can be established by choosing the comparators as $(\muvec^{*},\vvec^{*}) = 
(\muvec^{\pi^{*}},\vvec^{\overline{\pi}_{T}})\in\Delta_{\Atot}\times\Rn^{S}$. Indeed, as we show in 
\cref{appx:AMDP_gap} (a result adapted from \citealp{cheng2020reduction}), the two quantities under this choice can be 
related as
\begin{equation}\label{eq:sub_opt}
	\EE{G(\muvec^{\pi^{*}};\vvec^{\overline{\pi}_{T}})}
	= \EE{\iprod{\muvec^{\pi^{*}} - \muvec^{\overline{\pi}_{T}}}{\rvec}}.
\end{equation}

\subsection{Methodology}
In order to apply standard OMD to solve \cref{eq:AMDP}, previous LP-based 
approaches to planning in finite AMDPs \cite{MWang2017, jin2020efficiently} 
required the domain $\VV$ to cover $\vvec^{*}$, which requires prior knowledge of the 
properties of the MDP. To this end, they made the assumption that the value functions 
of all policies have bounded \emph{span seminorm}: for all policies $\pi$, the value function 
$\vvec^{\pi}$ satisfies $\spannorm{\vvec^\pi} = \max_s \vvec^\pi(s) - \min_{s'} \vvec^\pi(s') \le B$ for some $B>0$. 
We call this quantity the \emph{worst-case bias span}. A simple way to make sure that the above 
assumption holds is to suppose that the Markov chains induced by each policy $\pi$ have bounded 
\emph{mixing time} $t_{\text{mix}}$, defined as
\[
t_{\text{mix}} = \max_{\pi}\left[\argmin_{t\leq1}\ev{\max_{\nuvec\in\Delta_{S}} 
	\norm{\nuvec\transpose\pa{\Pm^{\pi}}^{t} - \nuvec^{\pi}}_{1}}\right].
\]
This ensures that the supremum norm of the value of any policy is 
bounded above with $\norm{\vvec^{\pi}}_{\infty}\leq 2t_{\text{mix}}$. 
Previous works of \citet{MWang2017, jin2020efficiently} assumed this mixing-time parameter 
to be known, and designed iterative algorithms that require projections to the set $\VV_B = \ev{\vvec\in\real^S :\,
\infnorm{\vvec} \le 2t_{\text{mix}}}$. Since this parameter is typically unknown and is hard to estimate, these 
algorithms are not fully satisfactory.

We are interested in near-optimal planning in general AMDPs for which the 
stationary state distribution is well defined and bias span is potentially 
unknown, and thus we have to set $\VV = \Rn^{S}$. 
Since the primal variables are naturally restricted to the simplex 
domain, we only require the stabilization trick to control the actions of the 
min-player in the bound. 
Hence, we can bound the duality gap against arbitrary 
comparator points $(\vvec^{*};\muvec^{*})$ as follows:
\begin{align}
	\EE{G(\vvec^{*};\muvec^{*})}
	\nonumber&= \EE{\LL(\ovvec_T;\muvec^{*}) - \LL(\vvec^{*};\omuvec_T)}\\
	\nonumber&\leq 
	\frac{1}{T}\sum_{t=1}^{T}
	\EE{\LL\pa{\vvec_{t};\muvec^{*}} - \LL\pa{\vvec^{*};\muvec_{t}}}\\
	\nonumber&= 
	\frac{1}{T}\sum_{t=1}^{T}
	\EE{\LL\pa{\vvec_{t};\muvec^{*}} - \LL\pa{\vvec_{t};\muvec_{t}}}
	+ 
	\frac{1}{T}\sum_{t=1}^{T}
	\EE{\LL\pa{\vvec_{t};\muvec_{t}}- \LL\pa{\vvec^{*};\muvec_{t}}}\\
	\nonumber&= 
	\frac{1}{T}\sum_{t=1}^{T}
	\EE{\LL\pa{\vvec_{t};\muvec^{*}} - \LL\pa{\vvec_{t};\muvec_{t}}}
	+ 
	\frac{1}{T}\sum_{t=1}^{T}
	\EE{\Lreg\pa{\vvec_{t};\muvec_{t}}- \Lreg\pa{\vvec^{*};\muvec_{t}}}\\
	\label{eq:reg_gap_AMDP}&\quad
	+\frac{\varrho_v}{T}\sum_{t=1}^{T}\EE{\HHv(\vvec^{*}) - \HHv(\vvec_{t})},
\end{align}
where we have defined $\Lreg\pa{\vvec;\muvec} = \LL\pa{\vvec;\muvec} + \varrho_v\HHv(\vvec)$. 

Taking into account the new (unregularized) loss objective of the max-player 
and required projections to the simplex, our algorithm executes \COMID to 
optimize $\vvec$ and standard OMD (which is same as \COMIDA with $\varrho_{\mu}=0$) for $\muvec$. Precisely, the 
updates are calculated by solving
	\begin{align*}
		\vvec_{t+1}
		&= \argmin_{\vvec\in\Rn^{S}}\ev{\iprod{\vvec}{\gtv}
			+ \varrho_v\infnorm{\vvec}^{2} + 
			\frac{1}{2\eta_{v}}\sqtwonorm{\vvec - 
			\vvec_{t}}}
	\\
		\muvec_{t+1}
		&= \argmin_{\muvec\in\Delta_{\Atot}}\ev{-\iprod{\muvec}{\gtmu}
			+ \frac{1}{\eta_{\mu}}\DDKL{\muvec}{\muvec_{t}}},
	\end{align*}
using the gradient estimators described in the main text.
We present the complete pseudocode as Algorithm~\ref{alg:OMDA_AMDP}.
\begin{algorithm}[H]
	\caption{\COMIDAMDP}\label{alg:OMDA_AMDP}
	\begin{algorithmic}
		\STATE {\bfseries Input:} Step sizes $\eta_{v},\eta_{\mu}$, 
		Regularization 
		constants $\varrho_{v}$, Initial points 
		$\vvec_{1}, \muvec_{1}$.
		
		\FOR{$t=1$ {\bfseries to} $T$}
		
		\STATE //Mirror Descent//
		\STATE Sample $(s_{t},a_{t})\sim\muvec_{t}, s'_{t}\sim 
		p(\cdot|s_{t},a_{t})$
		\STATE Compute $\gtv = \vec{e}_{s'_{t}} - \vec{e}_{s_{t}}$
		\STATE Update
		
		$\qquad\qquad \vvec_{t+1}	=
		\argmin_{\vvec\in\Rn^{S}}\ev{\iprod{\vvec}{\gtv}  + 
			\varrho_v \infnorm{\vvec}^2 + 
			\frac{1}{2\eta_{v}} \twonorm{\vvec - \vvec_t}^2}$
		
		\quad
		
		\STATE //Mirror Ascent//
		\STATE Sample $\overline{s}'_{t}\sim 
		p(\cdot|s,a)$ for all $(s,a)\in\AAtot$
		\STATE Compute
		$\gtmu = \sum_{(s,a)}[r(s,a) + v_{t}(\overline{s}'_{t}) - 
		v_{t}(s)]\vec{e}_{(s,a)}$
		\STATE Update
		
		$\qquad\qquad \muvec_{t+1} = 
		\argmin_{\muvec\in\Delta_{\Atot}}\ev{-\iprod{\muvec}{\gtmu}
			+ \frac{1}{\eta_{\mu}}\DDf{\text{KL}}{\muvec}{\muvec_{t}}}$
		
		\quad
		\ENDFOR
		\STATE {\bfseries Return}\quad $\overline{\vvec}_T = 
		\frac{1}{T}\sum_{t=1}^{T}\vvec_{t}$, 
		$\overline{\muvec}_T = \frac{1}{T}\sum_{t=1}^{T}\muvec_{t}$.
	\end{algorithmic}
\end{algorithm}

\section{The proof of Theorem~\ref{thm:res3}}
We restate the result here for convenience of the reader.
\begin{theorem}
Let $\varrho_{v}=4\eta_{\mu}$. Then, the 
output of \COMIDAMDP satisfies the 
following bound:
	\begin{align*}
		&\EE{\iprod{\muvec^{\pi^{*}} - 
		\muvec^{\overline{\pi}_{T}}}{\rvec}} 
		\leq \frac{\DDKL{\muvec^{\pi^{*}}}{\muvec_{1}}}{\eta_{\mu}T} + 
		\eta_{\mu}
		+
		\pa{\frac{1}{\eta_{v}T} + 4\eta_{\mu}}
		\EE{\sqtwonorm{\vvec^{\overline{\pi}_{T}}}}
		+ 2\eta_{v}
	\end{align*}
\end{theorem}

We start by stating a useful result (which we have learned from \citealp{cheng2020reduction}) that connects the duality 
gap with the suboptimality of the policy output by the algorithm. 
\begin{lemma}\label{appx:AMDP_gap}(cf. Proposition 4 of 
\citealp{cheng2020reduction})
	The duality gap at $(\omuvec_{T},\ovvec_{T})$ 
	satisfies
	\[
	G(\muvec^{\pi^{*}},\vvec^{\overline{\pi}_{T}})=
	\LL(\muvec^{\pi^{*}}; \ovvec_{T}) - 
	\LL(\omuvec_{T}; 
	\vvec^{\overline{\pi}_{T}}) = \rho^{*} - 
	\rho^{\overline{\pi}_{T}}.
	\]
\end{lemma}
\begin{proof}
	From \cref{eq:AMDP_gap}, recall that
	\begin{align}\label{eq:sub_opt2}
		G(\muvec^{\pi^{*}},\vvec^{\overline{\pi}_{T}})
		&=
		\LL(\muvec^{\pi^{*}}; \ovvec_{T}) - 
		\LL(\omuvec_{T}; \vvec^{\overline{\pi}_{T}}).
	\end{align}
	By definition of the Lagrangian, we can write
	\begin{equation*}
		\LL(\muvec^{\pi^{*}}; \ovvec_{T})
		= \iprod{\muvec^{\pi^{*}}}{\rvec} + 
		\iprod{\ovvec_{T}}{\Pm\transpose\muvec^{\pi^{*}} - 
			\Em\transpose\muvec^{\pi^{*}}}
		= \iprod{\muvec^{\pi^{*}}}{\rvec},
	\end{equation*}
	since $\muvec^{\pi^{*}}$ is a valid stationary distribution that satisfies $\Pm\transpose\muvec^{\pi^{*}} =	
\Em\transpose\muvec^{\pi^{*}}$. On the other hand, 
	using 
	that $\overline{\muvec}_T\in\Delta_{\Atot}$ and rearranging terms we have that:
	\begin{align*}
		\LL(\omuvec_{T}; \vvec^{\overline{\pi}_{T}})
		&= \iprod{\omuvec_{T}}{\rvec} + 
		\iprod{\vvec^{\overline{\pi}_{T}}}{\Pm\transpose\omuvec_{T} - 
			\Em\transpose\omuvec_{T}} + 
		\rho^{\overline{\pi}_{T}}(1-\iprod{\omuvec_{T}}{\vec{1}})\\
		&= \iprod{\omuvec_{T}}{\rvec + \Pm\vvec^{\overline{\pi}_{T}} - 
			\Em\vvec^{\overline{\pi}_{T}} - \rho^{\overline{\pi}_{T}}\vec{1}} + 
		\rho^{\overline{\pi}_{T}}\\
		&= 
		\sum_{s, a}\sum_{a'}\omuvec_{T}(s,a')\overline{\pi}_{T}(a|s)\Bpa{r(s,a)
		 + \iprod{p(\cdot|s,a)}{\vvec^{\overline{\pi}_{T}}} - 
		 \vvec^{\overline{\pi}_{T}}(s) - 
		 \rho^{\overline{\pi}_{T}} }  + 
		\rho^{\overline{\pi}_{T}} = \rho^{\overline{\pi}_{T}},
	\end{align*}
	where the last equality holds by definition of $\overline{\pi}_{T}$ in the 
	main text and the value functions in \cref{eq:valuefn}. Combining 
	both expressions in \cref{eq:sub_opt2} gives the desired result.
\end{proof}

\subsection{Proof of \cref{thm:res3}}\label{appx:res3}
First, we prove that the gradient norms are bounded. By definition of the 
gradients,
\begin{equation}\label{eq:gradv}
	\EEt{\sqtwonorm{\gtv}}
	= \EEt{\sqtwonorm{\vec{e}_{s'_{t}} - \vec{e}_{s_{t}}}}
	= \EEt{1 - 2\II{s'_{t}=s_{t}} + 1}
	\leq 2.
\end{equation}
Also, using that $r(s,a)\in[0,1]$ for any $(s,a)\in\AAtot$,
\begin{equation}\label{eq:gradmu}
\begin{split}
	\infnorm{\gtmu}^{2}
&\le \max_{(s,a,s')}\abs{r(s,a) + v_{t}(s') - v_{t}(s)}^{2}\le \pa{1 + 
2\infnorm{\vvec_{t}}}^{2} \le 2 + 8\infnorm{\vvec_{t}}^2,
\end{split}
\end{equation}
where the last inequality is Cauchy--Schwarz.

In what follows, we let $\vvec^* = \vvec^{\overline{\pi}_T}$, and derive a 
bound on the duality gap evaluated at this 
comparator point. We start by appealing to Lemma~\ref{appx:AMDP_gap} and 
decomposing the duality gap as follows:
\begin{align}
	\rho^{*} - \rho^{\overline{\pi}_{T}} &=\EE{G(\vvec^{*};\muvec^{*})}
	\nonumber\leq 
	\frac{1}{T}\sum_{t=1}^{T}
	\EE{\LL\pa{\vvec_{t},\muvec^{*}} - \LL\pa{\vvec_{t},\muvec_{t}}}\\
	\label{eq:gap_MDP}&\quad
	+ 
	\frac{1}{T}\sum_{t=1}^{T}
	\EE{\Lreg\pa{\vvec_{t},\muvec_{t}}- \Lreg\pa{\vvec^{*},\muvec_{t}}}\\
	&\nonumber\quad
	+\frac{\varrho_v}{T}\sum_{t=1}^{T}\EE{\infnorm{\vvec^{*}}^2 - 
	\infnorm{\vvec_{t}}^2}.
\end{align}
Let $\gmu=\nabla_{\muvec}\LL(\vvec_{t}; \muvec_{t})$ denote the gradient 
of the Lagrangian in round $t$. By the standard online 
mirror descent analysis, we obtain the following upper 
bound on the first term that corresponds to the regret of the $\mu$-player:
\begin{align}
	\nonumber\sum_{t=1}^{T}
	\EE{\LL\pa{\vvec_{t};\muvec^{*}} - \LL\pa{\vvec_{t};\muvec_{t}}}
	&\overset{(a)}{\leq} \sum_{t=1}^{T}\EE{\iprod{\muvec^{*} - 
			\muvec_{t}}{\gtmu}}\\
	\nonumber&\overset{(b)}{\leq} \frac{ 
		\DDKL{\muvec^{*}}{\muvec_{1}}}{\eta_{\mu}} + 
	\frac{\eta_{\mu}}{2}\sum_{t=1}^{T}\EE{\infnorm{\gtmu}^{2}}
	+\sum_{t=1}^{T}\EE{\iprod{\gmu - 
			\gtmu}{\muvec^{*} - \muvec_{t}}}
	\\
	\label{eq:muplayerbound} &\overset{(c)}{\leq} \frac{ 
		\DDKL{\muvec^{*}}{\muvec_{1}}}{\eta_{\mu}} + 
	\eta_\mu\sum_{t=1}^T\EE{\infnorm{\gtmu}^{2}}
	\\
	\nonumber &\overset{(d)}{\leq} \frac{ 
	\DDKL{\muvec^{*}}{\muvec_{1}}}{\eta_{\mu}} + 
\eta_\mu\sum_{t=1}^T\EE{1 + 
4\infnorm{\vvec_t}^2}
\end{align}
Here, we have used \emph{(a)} \cref{def:subgrad}, \emph{(b)} \cref{lem:COMID} 
with $\U = \Delta_{\Atot}$, $\lt\pa{\cdot} = -\LL\pa{\vvec_{t};\cdot}$, 
$\DDu{\uvec}{\uvec'} = \DDKL{\uvec}{\uvec'}$, $\varrho_{u}=0$ and 
$\uvec_{1}=\muvec_{1}$, $(c)$ that $\muvec^{*}$ is a fixed 
comparator and $\gtmu$ is an unbiased estimate of $\gmu$, as well as $(d)$ the 
bound on the gradient norm established in Equation~\eqref{eq:gradmu}. 

As for the second term that corresponds to the regret of the $v$-player, the 
analysis is somewhat more involved. One 
challenge is that the comparator point $\vvec^{*} = \vvec^{\overline{\pi}_{t}}$ 
is dependent on the iterates. To address this, we apply \cref{cor:COMID} with 
the appropriate parameters including $\U = \Rn^{X}$, $\omega_u\pa{u} = 
\frac{1}{2}\sqtwonorm{\uvec}$,
, as well as noting that $\omega_u$ is 
$1$-strongly convex and is the dual norm of itself gives the following bound:
\begin{align}
	\nonumber&\sum_{t=1}^{T}\EE{\Lreg\pa{\vvec_{t};\muvec_{t}}- 
		\Lreg\pa{\vvec^{*};\muvec_{t}}}\\
	\label{eq:vplayerbound}&\qquad\qquad\leq
	\frac{\EE{\sqtwonorm{\vvec^{*} - \vvec_{1}}}}{\eta_v}
	+ \eta_v\sum_{t=1}^{T}\EE{\sqopnorm{\gtv}{2}}\\
	\nonumber&\qquad\qquad\leq
	\frac{\EE{\sqtwonorm{\vvec^{*} - \vvec_{1}}}}{\eta_v}
	+ 2\eta_v.
\end{align}
The last inequality follows from using that $\EE{\sqopnorm{\gtv}{2}} \le 2$. 
Putting Equations~\eqref{eq:muplayerbound} and \eqref{eq:vplayerbound} in  
Equation~\eqref{eq:gap_MDP}, we finally obtain the following bound:
\begin{align*}
	\rho^{*} - \rho^{\overline{\pi}_{T}}
	&=\EE{G(\vvec^{*};\muvec^{*})}\\
	&\le \frac{ 
		\DDKL{\muvec^{*}}{\muvec_{1}}}{\eta_{\mu}T} + 
	\frac{\eta_\mu}{T}\sum_{t=1}^{T}\EE{1 + 4\infnorm{\vvec_t}^2} 
	+ 
	\frac{\EE{\twonorm{\vvec^* - \vvec_1}^2}}{\eta_v T} + 2\eta_v	
	+\frac{\varrho_v}{T}\sum_{t=1}^{T}\EE{\infnorm{\vvec^{*}}^2 - 
		\infnorm{\vvec_{t}}^2}.
\end{align*}
Recalling the choice $\vvec_1 = 0$, choosing $\varrho_v = 4\eta_\mu$, and 
bounding $\infnorm{\vvec^*} \le 
\twonorm{\vvec^*}$ we obtain the result claimed in the theorem.\qed

\newpage
\section{Auxiliary Lemmas}\label{appx:aux}
\begin{lemma}(cf.~Theorem 8 of \citealp{duchi2010composite})\label{lem:COMID}
	Let $\ell_{t}:\U\rightarrow\Rn$ be convex, 
	$\gu\in\partial\ell_{t}(\uvec_{t})$ and $\gtu$ be such that $\EEt{\gtu} = 
	\gu$. Given $\uvec_{1}\in\U$, define 
	$\bm{\tilde{g}}_{u}(1)\in\Rn^{m}$ and the sequence of vectors 
	$\{(\uvec_{t}, \gtu)\}_{t=2}^{T}$ via the following recursion for $t\in[T]$:
	\begin{equation}\label{eq:temp}
		\uvec_{t+1}
		= \argmin_{\uvec\in\U}\ev{\iprod{\uvec}{\gtu} + 
		\varrho_{u}\HHu\pa{\uvec} + \frac{1}{\eta_{u}}\DDu{\uvec}{\uvec_{t}}}.
	\end{equation}
	Suppose the distance-generating function 
	$\omega_u$ is $\gamma_u$-strongly convex with respect to 
	$\norm{\cdot}_{u}$. For any $\uvec^{*}\in\U$,
	\begin{align*}
		&\sum_{t=1}^{T}\EE{\ltreg\pa{\uvec_{t}} - 
			\ltreg\pa{\uvec^{*}}}\\
		&\qquad\qquad\leq
		\frac{\EE{\DDu{\uvec^{*}}{\uvec_{1}}}}{\eta_u}
		+ \frac{\eta_u}{2\gamma_u}\sum_{t=1}^{T}\EE{\sqopnorm{\gtu}{u,*}}
		+ \varrho_u\EE{\HHu(\uvec_{1})} + \sum_{t=1}^{T}\EE{\iprod{\gu - 
		\gtu}{\uvec_{t} - \uvec^{*}}}.
	\end{align*}
	where
	\begin{equation}\label{eq:regloss}
		\ltreg\pa{\uvec} = \ell_{t}\pa{\uvec} + \varrho_{u}\HHu\pa{\uvec}.
	\end{equation}
\end{lemma}
\begin{proof}	
	Using the definition of $\ltreg$, consider the 
	regret in terms of the regularized loss:
	\begin{align*}
		\ltreg\pa{\uvec_{t}} - \ltreg\pa{\uvec^{*}}
		&= \lt\pa{\uvec_{t}} - \lt\pa{\uvec^{*}}
		+ \varrho_u\HHu(\uvec_{t})
		- \varrho_u\HHu(\uvec^{*})\\
		&= \bpa{\lt\pa{\uvec_{t}} - \lt\pa{\uvec^{*}}
			+ \varrho_u\HHu(\uvec_{t+1})
			- \varrho_u\HHu(\uvec^{*})}
		+ \varrho_u\Bpa{\HHu(\uvec_{t}) - \HHu(\uvec_{t+1})}.
	\end{align*}
To proceed, we let $\Hu{t+1} \in \partial \HHu(\uvec_{t+1})$, so that we can use the convexity of $\ell_t$ and $\HHu$ 
to bound the first set of terms as 
	\begin{align}
		\nonumber&{\lt\pa{\uvec_{t}} - \lt\pa{\uvec^{*}}
			+ \varrho_u\HHu(\uvec_{t+1})
			- \varrho_u\HHu(\uvec^{*})}\\
		\nonumber&\quad\quad\leq \iprod{\gu}{\uvec_{t} - \uvec^{*}} + 
		\varrho_u\iprod{\Hu{t+1}}{\uvec_{t+1} - \uvec^{*}}\\
		\label{eq:regu}&\quad\quad= \iprod{\gtu}{\uvec_{t} - \uvec^{*}} + 
		\varrho_u\iprod{\Hu{t+1}}{\uvec_{t+1} - \uvec^{*}}
		+ \iprod{\gu - \gtu}{\uvec_{t} - \uvec^{*}}.
	\end{align}
	Before we proceed to bound the first two terms, note that $\uvec_{t+1}$ in 
	\cref{eq:temp} is a solution to a constrained convex optimization problem, and as 
	a result it satisfies the following optimality condition for any $\uvec \in \U$:
	\begin{align}
		\label{eq:optcondu}&\iprod{\uvec - \uvec_{t+1}}{\gtu + 
		\varrho_u\Hu{t+1} + 
			\frac{1}{\eta_u}\pa{\nabla\omega_u(\uvec_{t+1}) - 
				\nabla\omega_u(\uvec_{t})}} \geq 0.
	\end{align}
	Thus, we bound the first two terms on the right-hand side of the 
	inequality~\eqref{eq:regu} as follows:
	\begin{align*}
		&\iprod{\gtu}{\uvec_{t} - \uvec^{*}} + 
		\varrho_u\iprod{\Hu{t+1}}{\uvec_{t+1} - \uvec^{*}}\\
		&\qquad\qquad\qquad\qquad\qquad=
		\iprod{\gtu + \varrho_u\Hu{t+1}}{\uvec_{t+1} - \uvec^{*}}
		+ \iprod{\gtu}{\uvec_{t} - \uvec_{t+1}}\\
		&\qquad\qquad\qquad\qquad\qquad=
		\iprod{\gtu + \varrho_u\Hu{t+1} + 
			\frac{1}{\eta_u}\pa{\nabla\omega_u(\uvec_{t+1}) - 
				\nabla\omega_u(\uvec_{t})}}{\uvec_{t+1} - \uvec^{*}}\\
		&\qquad\qquad\qquad\qquad\qquad\quad
		+ \frac{1}{\eta_u}\iprod{\nabla\omega_u(\uvec_{t+1}) - 
			\nabla\omega_u(\uvec_{t})}{\uvec^{*} - \uvec_{t+1}}
		+ \iprod{\gtu}{\uvec_{t} - \uvec_{t+1}}\\
		&\qquad\qquad\qquad\qquad\qquad\overset{(a)}{\leq}
		\frac{1}{\eta_u}\iprod{\nabla\omega_u(\uvec_{t+1}) - 
			\nabla\omega_u(\uvec_{t})}{\uvec^{*} - \uvec_{t+1}}
		+ \iprod{\gtu}{\uvec_{t} - \uvec_{t+1}}\\
		&\qquad\qquad\qquad\qquad\qquad\overset{(b)}{=}
		\frac{1}{\eta_u}\Bpa{\DDu{\uvec^{*}}{\uvec_{t}} - 
			\DDu{\uvec^{*}}{\uvec_{t+1}}}
		- \frac{1}{\eta_u}\DDu{\uvec_{t+1}}{\uvec_{t}}
		+ \iprod{\gtu}{\uvec_{t} - \uvec_{t+1}}\\
		&\qquad\qquad\qquad\qquad\qquad\overset{(c)}{\leq}
		\frac{1}{\eta_u}\Bpa{\DDu{\uvec^{*}}{\uvec_{t}} - 
			\DDu{\uvec^{*}}{\uvec_{t+1}}}
		- \frac{\gamma_u}{2\eta_u}\sqopnorm{\uvec_{t+1} - \uvec_{t}}{u}
		+ \iprod{\gtu}{\uvec_{t} - \uvec_{t+1}}\\
		&\qquad\qquad\qquad\qquad\qquad\leq
		\frac{1}{\eta_u}\Bpa{\DDu{\uvec^{*}}{\uvec_{t}} - 
			\DDu{\uvec^{*}}{\uvec_{t+1}}}
		+ \frac{\gamma_u}{\eta_u}\sup_{\uvec}
		\pa{\iprod{\frac{\eta_u}{\gamma_u}\gtu}{\uvec} - 
			\frac{1}{2}\sqopnorm{\uvec}{u}}\\
		&\qquad\qquad\qquad\qquad\qquad\overset{(d)}{=}
		\frac{1}{\eta_u}\Bpa{\DDu{\uvec^{*}}{\uvec_{t}} - 
			\DDu{\uvec^{*}}{\uvec_{t+1}}}
		+ \frac{\gamma_u}{2\eta_u}\sqopnorm{\frac{\eta_u}{\gamma_u}\gtu}{u,*}\\
		&\qquad\qquad\qquad\qquad\qquad=
		\frac{1}{\eta_u}\Bpa{\DDu{\uvec^{*}}{\uvec_{t}} - 
			\DDu{\uvec^{*}}{\uvec_{t+1}}}
		+ \frac{\eta_u}{2\gamma_u}\sqopnorm{\gtu}{u,*}.
	\end{align*}
	We have used $(a)$ the optimality condition stated in \cref{eq:optcondu}, 
	$(b)$ 
	the so-called \emph{three-points identity} of Bregman divergences 
	(cf.~Lemma~4.1 in \cite{beck2003mirror}), $(c)$ 
	the strong convexity of $\DDu{\cdot}{\uvec_t}$ and $(d)$ the fact that for 
	any 
	norm $\norm{\cdot}$, we have 
	$\sup_{\uvec} \ev{\iprod{\uvec}{\gvec} - \frac 12 \norm{\uvec}^2} = \frac 
	12 
	\norm{\gvec}_*^2$. 
	
	Thus, putting together 
	all the above calculations,  we 
	arrive at 
	the following bound:
	\begin{align*}
		\nonumber&\lt\pa{\uvec_{t}} - \lt\pa{\uvec^{*}}
			+ \varrho_u\HHu(\uvec_{t+1})
			- \varrho_u\HHu(\uvec^{*})\\
		&\qquad\qquad\leq \iprod{\gtu}{\uvec_{t} - \uvec^{*}} + 
		\varrho_u\iprod{\Hu{t+1}}{\uvec_{t+1} - \uvec^{*}}
		+ \iprod{\gu - \gtu}{\uvec_{t} - \uvec^{*}}\\
		&\qquad\qquad\leq\frac{1}{\eta_u}\Bpa{\DDu{\uvec^{*}}{\uvec_{t}} - 
			\DDu{\uvec^{*}}{\uvec_{t+1}}}
		+ \frac{\eta_u}{2\gamma_u}\sqopnorm{\gtu}{u,*} + 
		\iprod{\gu - \gtu}{\uvec_{t} - \uvec^{*}}.
	\end{align*}
	Furthermore, plugging in the definition of $\ltreg$ we get 
	\begin{align*}
		\ltreg\pa{\uvec_{t}} - \ltreg\pa{\uvec^{*}}
		&= \pa{\lt\pa{\uvec_{t}} - \lt\pa{\uvec^{*}}
			+ \varrho_u\HHu(\uvec_{t+1})
			- \varrho_u\HHu(\uvec^{*})}
		+ \varrho_u\Bpa{\HHu(\uvec_{t}) - \HHu(\uvec_{t+1})}\\
		&\leq\frac{1}{\eta_u}\Bpa{\DDu{\uvec^{*}}{\uvec_{t}} - 
			\DDu{\uvec^{*}}{\uvec_{t+1}}}
		+ \frac{\eta_u}{2\gamma_u}\sqopnorm{\gtu}{u,*} + 
		\iprod{\gu - \gtu}{\uvec_{t} - \uvec^{*}}\\
		&\quad
		+ \varrho_u\Bpa{\HHu(\uvec_{t}) - \HHu(\uvec_{t+1})}.
	\end{align*}
	Hence, taking marginal expectations on both sides, summing over $t=1,\cdots,T$ steps, 
	evaluating the telescoping terms and upper bounding some
	negative terms by zero, we finally obtain the following bound on the total regret of \COMID on the 
	regularized objective:
	\begin{align*}
		&\sum_{t=1}^{T}\EE{\ltreg\pa{\uvec_{t}} - 
		\ltreg\pa{\uvec^{*}}}\\
		&\quad\leq
		\frac{\EE{\DDu{\uvec^{*}}{\uvec_{1}}}}{\eta_u}
		+ \frac{\eta_u}{2\gamma_u}\sum_{t=1}^{T}\EE{\sqopnorm{\gtu}{u,*}}
		+\sum_{t=1}^{T}\EE{\iprod{\gu - \gtu}{\uvec_{t} - \uvec^{*}}}
		+ \varrho_u\EE{\HHu(\uvec_{1})}.
	\end{align*}

	This completes the proof.
\end{proof}

\begin{corollary}\label{cor:COMID}
	Suppose the sequence of vectors $\{(\uvec_{t}, \gtu)\}_{t=1}^{T}$ are as 
	described in \cref{lem:COMID} above. If the comparator $\uvec^{*}$ is 
	adaptive and potentially dependent on the interaction history $\F_T$, the 
	following inequality holds:
	\[
	\sum_{t=1}^{T}\EE{\ltreg\pa{\uvec_{t}} - 
		\ltreg\pa{\uvec^{*}}}
	\leq
	\frac{2\EE{\DDu{\uvec^{*}}{\uvec_{1}}}}{\eta_u}
	+ \frac{\eta_u}{\gamma_u}\sum_{t=1}^{T}\EE{\sqopnorm{\gtu}{u,*}}
	+ \varrho_u\EE{\HHu(\uvec_{1})}.
	\]
\end{corollary}
\begin{proof}
	Recall that for any comparator $\uvec^{*}\in\U$ we can upper-bound the 
	total regret from \cref{lem:COMID} as:
	\begin{align}\label{eq:comida_regret}
		\nonumber&\sum_{t=1}^{T}\EE{\ltreg\pa{\uvec_{t}} - 
			\ltreg\pa{\uvec^{*}}}\\
		&\quad\leq
		\frac{\EE{\DDu{\uvec^{*}}{\uvec_{1}}}}{\eta_u}
		+ \frac{\eta_u}{2\gamma_u}\sum_{t=1}^{T}\EE{\sqopnorm{\gtu}{u,*}}
		+ \varrho_u\EE{\HHu(\uvec_{1})}
		+\sum_{t=1}^{T}\EE{\iprod{\gu - \gtu}{\uvec_{t} - \uvec^{*}}}.
	\end{align}
	
		Notice the last sum is zero when the comparator 
		point $\uvec^*$ is independent of the interaction history 
		$\F_T$, but this is generally not true for 
		adaptively chosen comparators. A crude upper bound may be proved by 
		treating it as a supremum of martingales, for 
		instance by relying on the techniques of 
		\citet{rakhlin2017equivalence}. Inspired by their techniques, we 
		provide a refined bound on this term based on a reduction to online 
		learning, which has otherwise been first introduced in this 
		context by \citet{nemirovski2009robust}. \cref{lem:aux_regret} states 
		the resulting bound which we highlight below:
		\[
		\sum_{t=1}^T \EE{\iprod{\gu - \gtu}{\uvec_{t} - \uvec^*}} 
		\le \frac{\EE{\DDu{\uvec^*}{\uvec_{1}}}}{\eta_u}
		+ 
		\frac{\eta_u}{2\gamma_u}\sum_{t=1}^T\EEt{\sqopnorm{\gu - 
				\gtu}{u,*}}.
		\]
		Notably, the terms 
		appearing on the right-hand side are 
		comparable to the terms appearing in the bound of 
		Equation~\eqref{eq:comida_regret}. In particular, using that 
		\[
		\EEt{\sqopnorm{\gu - \gtu}{u,*}} \le \EEt{\sqopnorm{\gtu}{u,*}}
		\]
		holds because of the unbiasedness of the gradient estimate $\gtu$, we 
		can combine the bound of 
		Lemma~\ref{lem:aux_regret} and Equation~\eqref{eq:comida_regret} to 
		obtain the bound
		\[
		\sum_{t=1}^{T}\EE{\ltreg\pa{\uvec_{t}} - 
			\ltreg\pa{\uvec^{*}}}
		\leq
		\frac{2\EE{\DDu{\uvec^{*}}{\uvec_{1}}}}{\eta_u}
		+ \frac{\eta_u}{\gamma_u}\sum_{t=1}^{T}\EE{\sqopnorm{\gtu}{u,*}}
		+ \varrho_u\EE{\HHu(\uvec_{1})}.
		\]
	
	This completes the proof.
\end{proof}

\begin{lemma}\label{lem:aux_regret} Let 
	$\uvec^*\in\U$ be an arbitary adaptive comparator point as described in 
	the statement of \cref{cor:COMID}. Then,
	\begin{align*}
		&\sum_{t=1}^T \EE{\iprod{\gu - \gtu}{\uvec_{t} - \uvec^*}} 
		\\
		&\qquad\qquad \le \frac{\EE{\DDu{\uvec^*}{\uvec_{1}}}}{\eta_u}
		+ \frac{\eta_u}{2\gamma_u}\sum_{t=1}^T\EEt{\sqopnorm{\gu - 
				\gtu}{u,*}}.
	\end{align*}
\end{lemma}

	\subsection{The proof of Lemma~\ref{lem:aux_regret}}\label{app:aux_regret}
	For the sake of the proof, we will introduce an auxiliary online learning 
	game, 
	where in each 
	round $t=2,\dots,T$, the following steps are repeated:
	\begin{enumerate}
		\item the online learner chooses $\uhat_{t} \in \U$,
		\item the environment chooses the cost function $\cvec_t = \pa{\gu - 
			\gtu}$,
		\item the online learner incurs cost $\iprod{\cvec_t}{\uhat_t}$ and 
		observes $\cvec_t$.
	\end{enumerate}
	Having introduced this scheme, we can write the following decomposition: 
	\begin{align*}
		\sum_{t=1}^T \iprod{\gu - \gtu}{\uvec_{t} - \uvec^{*}}
		&= \sum_{t=1}^T \iprod{\gu - \gtu}{\uvec_{t} - \uhat_{t}}
		+ \sum_{t=1}^T \iprod{\gu - \gtu}{\uhat_{t} -\uvec^{*}}\\
		&= \sum_{t=1}^T \iprod{\gu - \gtu}{\uvec_{t} - \uhat_{t}}
		+ \sum_{t=1}^T \iprod{\cvec_{t}}{\uhat_{t} -\uvec^{*}},
	\end{align*}
	where the first sum is a martingale and the second term is the regret in 
	the auxiliary online learning game 
	we have just introduced. Since the expectation of the first term is zero, 
	this 
	leaves us with bounding the auxiliary regret.

	To this end, we will let the online learner run mirror descent with the 
	regularizer $\mathcal{D}_u$, learning rate 
	$\eta_u$ and initial iterate $\uhat_{1} = \uvec_{1}$. For the subsequent 
	rounds, the updates are
	given as
	$\uhat_t = 
	\argmin_{\uvec\in\U}\!\ev{\iprod{\uvec}{\cvec_{t-1}}
		+ \frac{1}{\eta_{u}}\DDu{\uvec}{\uhat_{t-1}}}$.
	By following the same steps as in \cref{lem:COMID} (with the 
	simplification $\varrho_u=0$) and noting that the comparator $\uvec^{*}$ is 
	independent of the interaction history of this auxiliary online learning 
	game, we obtain the following bound on the instantaneous regret in round 
	$t$:
	\begin{align*}
		\iprod{\cvec_{t}}{\uhat_{t} - \uvec^{*}}
		&\le
		\frac{1}{\eta_u}\Bpa{\DDu{\uvec^{*}}{\uhat_{t}} - 
			\DDu{\uvec^{*}}{\uhat_{t+1}}}
		+ \frac{\eta_u}{2\gamma_u}\sqopnorm{\gu - \gtu}{u,*}.
	\end{align*}
	Taking the marginal expectation on both sides and summing up over all 
	rounds $T$, concludes the proof.

\qed

\end{document}